\documentclass[accepted]{article}

\usepackage{microtype}
\usepackage{graphicx}
\usepackage{booktabs} 

\usepackage[nohyperref]{icml2019}

\icmltitlerunning{MONK --  Outlier-Robust Mean Embedding Estimation by Median-of-Means}

\usepackage{amsmath,amssymb,amsthm,enumitem,url,dsfont} 
\usepackage[mathscr]{eucal}
\usepackage[lofdepth,lotdepth]{subfig}   
\usepackage{float}
\newfloat{algorithm}{t}{lop}

\newcommand{\pscal}[2]{\left\langle #1, #2 \right\rangle}
\def\rme{\mathrm{e}}

\newcommand{\R}{\mathbb{R}}
\renewcommand{\P}{\mathbb{P}}

\newcommand{\E}{\mathbb{E}}
\newcommand{\Z}{\mathbb{Z}}
\newcommand{\norm}[1]{\left\|#1\right\|}
\newcommand{\cro}[1]{\left[#1\right]}

\newcommand{\set}[1]{\left\{#1\right\}}

\newcommand{\psh}[2]{\left\langle#1,#2\right\rangle}
\newcommand{\MOM}[2]{\text{MON}_{#1}\cro{#2}}

\newcommand{\X}{\mathscr{X}} 
\renewcommand{\k}{K} 
\newcommand{\Hk}{\mathscr{H}_{\k}} 
\newcommand{\tb}{\textbf} 
\renewcommand{\d}{\mathrm{d}} 
\renewcommand{\xi}{x} 
\newcommand{\med}[1]{\underset{#1}{\text{med}}}
\newcommand{\medi}[1]{\text{med}_{#1}} 
\renewcommand{\L}{\mathscr{L}} 
\newcommand{\G}{\mathscr{G}} 
\newcommand{\I}{\mathbb{I}} 
\newcommand{\Q}{\mathbb{Q}} 
\renewcommand{\b}{\mathbf} 
\newcommand{\MMD}{\operatorname{MMD}}
\newcommand{\MMDhat}{\widehat{\operatorname{MMD}}}
\newcommand{\MMDQ}{\MMD_{Q}}
\newcommand{\MMDQhat}{\widehat{\MMD}_{Q}}
\renewcommand{\O}{\mathscr{O}}

\DeclareMathOperator{\argmin}{\rm{argmin}}
\DeclareMathOperator{\argmax}{\textrm{argmax}}
\DeclareMathOperator{\Tr}{\operatorname{Tr}}  
  
\newtheorem{theorem}{Theorem}
\newtheorem{lemma}[theorem]{Lemma}

\usepackage{natbib}

\begin{document}

\twocolumn[
\icmltitle{MONK --  Outlier-Robust Mean Embedding Estimation by Median-of-Means}



\icmlsetsymbol{equal}{*}

\begin{icmlauthorlist}
\icmlauthor{Matthieu Lerasle}{T,M}
\icmlauthor{Zolt{\'a}n Szab{\'o}}{Z}
\icmlauthor{Timoth{\'e}e Mathieu}{T}
\icmlauthor{Guillaume Lecu{\'e}}{G}
\end{icmlauthorlist}

\icmlaffiliation{M}{CNRS, Universit{\'e} Paris Saclay, France}
\icmlaffiliation{Z}{CMAP, {\'E}cole Polytechnique, Palaiseau, France}
\icmlaffiliation{T}{Laboratoire de Math{\'e}matiques d'Orsay, Univ.\ Paris-Sud, France}
\icmlaffiliation{G}{CREST ENSAE ParisTech, France}

\icmlcorrespondingauthor{Matthieu Lerasle}{matthieu.lerasle@math.u-psud.fr}
\icmlcorrespondingauthor{Zolt{\'a}n Szab{\'o}}{zoltan.szabo@polytechnique.edu}

\icmlkeywords{mean embedding, maximum mean discrepancy, adversarial learning, median-of-means, robust statistics, reproducing kernel Hilbert space}

\vskip 0.3in
]



\printAffiliationsAndNotice{}  

\begin{abstract}
Mean embeddings provide an extremely flexible and powerful tool in machine learning and statistics to represent probability distributions and define
a semi-metric (MMD, maximum mean discrepancy; also called N-distance or energy distance), with numerous successful applications.
The representation is constructed as the expectation of the feature map defined by a kernel. As a mean, its classical empirical estimator, however, can be arbitrary severely affected even by a single outlier
in case of unbounded features. To the best of our knowledge, unfortunately even the consistency of the existing few techniques trying to alleviate this serious sensitivity bottleneck is unknown. In this paper, we show how the
recently emerged principle of median-of-means can be used to design estimators for kernel mean embedding and MMD with excessive resistance properties to outliers, and
optimal sub-Gaussian deviation bounds under mild assumptions.
\end{abstract}

\section{Introduction}
Kernel methods \cite{aronszajn50theory} form the backbone of a tremendous number of successful applications in machine learning thanks to their power in capturing complex relations 
\cite{scholkopf02learning,steinwart08support}.
The main idea behind these techniques is to map the data points to a feature space (RKHS, reproducing kernel Hilbert space) determined by the kernel, and apply linear methods in the feature space,
without the need to explicitly compute the map.

One crucial component contributing to this flexibility and efficiency (beyond the solid theoretical foundations) is the versatility of domains where kernels exist; examples include
trees \cite{collins01convolution,kashima02kernels},  time series \cite{cuturi11fast}, strings \cite{lodhi02text}, mixture models, hidden Markov models or linear dynamical systems \cite{jebara04probability},
sets \cite{haussler99convolution,gartner02multi}, fuzzy domains \cite{guevara17cross}, distributions \cite{hein05hilbertian,martins09nonextensive,muandet12learning}, groups \cite{cuturi05semigroup} such as
specific constructions on permutations \cite{jiao16kendall}, or graphs \cite{vishwanathan10graph,kondor16multiscale}.

Given a kernel-enriched domain $(\X,\k)$ one can represent probability distributions on $\X$ 
as a mean
\begin{align*}
  \mu_\P &= \int_{\X} \varphi(x) \d \P(x) \in \Hk, \quad \varphi(x) := \k(\cdot,x),
\end{align*}
which is a point in the RKHS determined by $\k$. This representation called \emph{mean embedding} \cite{berlinet04reproducing,smola07hilbert}
induces  a semi-metric\footnote{\citet{fukumizu08kernel,sriperumbudur10hilbert} provide conditions when MMD is a metric, i.e.\ $\mu$ is injective.} on
distributions called maximum mean discrepancy (MMD) \cite{smola07hilbert,gretton12kernel}
\begin{align}
    \text{MMD}(\P,\Q) = \left\|\mu_\P - \mu_\Q\right\|_{\Hk}. \label{eq:MMD}
\end{align}
With appropriate choice of the kernel, classical integral transforms widely used in probability theory and statistics can be recovered by
$\mu_\P$; for example, if $\X$ equipped with the scalar product $\pscal{\cdot}{\cdot}$ is a Hilbert space, the kernel $\k(x,y) = \rme^{\pscal{x}{y}}$ gives the moment-generating function,  $\k(x,y) = \rme^{\gamma \left\|x-y\right\|_2^2}$ ($\gamma>0$) the Weierstrass transform.
As it has been shown \cite{sejdinovic13equivalence} energy distance \cite{baringhaus04new,szekely04testing,szekely05new}---also known as
N-distance \cite{zinger92characterization,klebanov05n-distance} in the statistical literature---coincides with MMD.

Mean embedding and maximum mean discrepancy have been applied successfully, in kernel Bayesian inference \cite{song11kernel,fukumizu13kernel}, approximate Bayesian computation \cite{park16k2abc},
model criticism \cite{loyd14automatic,kim16examples}, two-sample \cite{baringhaus04new,szekely04testing,szekely05new,harchaoui07testing,gretton12kernel} or its differential private variant \cite{raj18differentially}, 
independence \cite{gretton08kernel,pfister17kernel} 
and goodness-of-fit testing \cite{jitkrittum17linear,balasubramanian17optimality}, domain adaptation \cite{zhang13domain} and generalization \cite{blanchard17domain}, change-point detection \cite{harchaoui07retrospective}, 
probabilistic programming \cite{scholkopf15computing}, post selection inference \cite{yamada18post},
distribution classification \cite{muandet12learning,zaheer17deep} and regression \cite{szabo16learning,law18bayesian}, causal discovery \cite{mooij16distinguishing, pfister17kernel}, generative adversarial networks 
\cite{dziugaite15training,li15generative,binkowski18demistifying}, understanding the dynamics of complex dynamical systems \cite{klus18eigendecompositions,klus19kernel}, or topological data analysis \cite{kusano16persistence}, among many others; \citet{maundet17kernel} provide a recent in-depth review on the topic.

Crucial to the success of these applications is the efficient and robust approximation of the mean embedding and MMD.
As a mean, the most natural approach to estimate $\mu_\P$ is the empirical average. Plugging this estimate into Eq.~\eqref{eq:MMD} produces directly an approximation of MMD, which can also be made unbiased
(by a small correction) or approximated recursively. These are the V-statistic, U-statistic and online approaches \cite{gretton12kernel}.
Kernel mean shrinkage estimators \cite{muandet16kernel} represent an other successful direction: they improve the efficiency of the mean embedding estimation by taking into account the Stein phenomenon.
Minimax results have recently been established: the optimal rate of mean embedding estimation  given $N$ samples from $\P$ is $N^{-1/2}$ 
\cite{tolstikhin17minimax} for
discrete measures and the class of measures with infinitely differentiable density when $\k$ is a continuous, shift-invariant kernel on $\X = \R^d$. For MMD,
using $N_1$ and $N_2$ samples from $\P$ and $\Q$, it is  $N_1^{-1/2} + N_2^{-1/2}$ \cite{tolstikhin16minimax} in case of radial universal kernels defined on $\X = \R^d$.

A critical property of an estimator is its robustness to contaminated data, outliers which are omnipresent in currently available massive and heterogenous datasets.
To the best of our knowledge, systematically \emph{designing outlier-robust mean embedding and MMD estimators} has hardly been touched in the literature; this is the focus of the current paper.
The issue is particularly serious in case of unbounded kernels when for example even a single outlier can ruin completely a classical empirical average based estimator. Examples for
unbounded kernels are the exponential kernel (see the example above about moment-generating functions), polynomial kernel, string, time series or graph kernels.

Existing related techniques comprise robust kernel density estimation (KDE) \cite{kim12robust}: the authors elegantly combine ideas from the KDE and M-estimator literature to
arrive at a robust KDE estimate of density functions. They assume that the underlying smoothing kernels\footnote{Smoothing kernels extensively studied in the non-parametric statistical literature \cite{gyorfi2002} are
assumed to be non-negative functions integrating to one.} are shift-invariant on $\X=\R^d$ and reproducing, and interpret KDE as a weighted mean in $\Hk$. The idea has been (i) adapted to construct
outlier-robust covariance operators in RKHSs in the context of kernel canonical correlation analysis
\cite{alam18influence}, and (ii) relaxed to general Hilbert spaces \cite{sinova18m-estimators}. Unfortunately, the consistency of the investigated empirical M-estimators is unknown, except
for finite-dimensional feature maps \cite{sinova18m-estimators}, or as density function estimators \cite{vandermeulen13consistency}.

To achieve our goal, we leverage the idea of Median-Of-meaNs (MON). Intuitively, MONs replace the linear operation of expectation with the median of
averages taken over non-overlapping blocks of the data, in order to get a robust estimate thanks to the median step.
MONs date back to \citet{MR855970, MR1688610,MR702836} for the estimation of the mean of real-valued random variables.
Their concentration properties have been recently studied by \citet{MR3576558, Minsker-Strawn2017} following the approach of \citet{MR3052407} for M-estimators. These studies focusing on the 
estimation of the mean of real-valued random variables are important as they can be used to tackle more general prediction problems in learning theory via the classical empirical risk minimization approach 
\cite{MR1719582} or by more sophisticated approach such as the minmax procedure \cite{MR2906886}.

In parallel to the minmax approach, there have been several attempts to extend the usage of MON estimators from $\R$ to more
general settings. For example, \citet{MR3378468, Minsker-Strawn2017} consider the problem of estimating the mean of a Banach-space valued random variable using ``geometrical'' MONs.
The estimators constructed by \citet{MR3378468, Minsker-Strawn2017} are computationally tractable  but the deviation bounds are suboptimal compared to those one can prove for the empirical mean under sub-Gaussian assumptions.
In regression problems, \citet{LugosiMendelson2016,LL1} proposed to combine the classical MON estimators on $\R$ in a ``test'' procedure that
can be seen as a Le Cam test estimator \cite{MR0334381}. The achievement in \cite{LugosiMendelson2016, LL1} is that they were able to obtain optimal deviation bounds for the resulting estimator using the powerful 
so-called small-ball method of \citet{MR3431642,Shahar-COLT}.
This approach was then extended to mean estimation $\R^d$ by \citet{LugosiMendelson2017-2} providing the first rate-optimal sub-Gaussian deviation bounds under minimal $L^2$-assumptions.
The constants of \citet{LugosiMendelson2016, LL1,LugosiMendelson2017-2} have been improved by \citet{Catoni-Giullini2017} for
the estimation of the mean in $\R^d$ under $L^4$-moment assumption and in least-squares regression under $L^4/L^2$-condition that is stronger than the small-ball assumption used by \citet{LugosiMendelson2016, LL1}.
Unfortunately, these estimators are computationally intractable; their risk bounds however serve as an important baseline for computable estimators such as the minmax MON estimators in regression \cite{LL2}.

Motivated by the computational intractability of the tournament procedure underlying the first rate-optimal sub-Gaussian deviation bound holding under minimal assumptions in $\R^{d}$ \citep{LugosiMendelson2017-2}, \citet{hopkins18mean} proposed a convex relaxation with polynomial, $O(N^{24})$ complexity where $N$ denotes the sample size.
\citet{cherapanamjeri19fast} have recently designed an alternative convex relaxation requiring $O(N^4+d N^2)$ computation which is still rather restrictive for large sample size and infeasible in infinite dimension.

Our goal is to extend the theoretical insight of  \citet{LugosiMendelson2017-2}
from $\R^d$ to kernel-enriched domains. Particularly,  we prove optimal sub-Gaussian deviation bounds for MON-based mean estimators in RKHS-s which hold
under minimal second-order moment assumptions. In order to achieve this goal, we use a different (minmax \citep{MR2906886,LL2}) construction which combined with properties specific to RKHSs (the mean-reproducing property of mean embedding and the integral probability metric representation of MMD) give rise to our practical MONK procedures. Thanks to the usage of medians the MONK estimators are also robust to  contamination.

Section~\ref{sec:defs-problem-formulation} contains  definitions and problem formulation. Our main results are given in Section~\ref{sec:results}.
Implementation of the MONK estimators is the focus of Section~\ref{sec:MONK-implementation}, with numerical illustrations in Section~\ref{sec:numerical-demos}.

\section{Definitions \& Problem Formulation} \label{sec:defs-problem-formulation}
In this section, we formally introduce the goal of our paper.

\tb{Notations:}
$\Z^+$ is the set of positive integers. $[M]:=\{1,\ldots,M\}$, $u_{S} := (u_m)_{m\in S}$, $S\subseteq [M]$. For a set $S$, $|S|$ denotes its cardinality. $\E$ stands for expectation. $\medi{q\in [Q]}\left\{z_q\right\}$ is the median of the $(z_q)_{q\in[Q]}$ numbers. Let
$\X$ be a separable topological space endowed with the Borel $\sigma$-field, $x_{1:N}$ denotes a sequence of i.i.d.\ random variables on $\X$ with law $\P$ (shortly, $x_{1:N}\sim \P$).
$\k:\X \times \X \rightarrow \R$ is a continuous (reproducing) kernel on $\X$, $\Hk$ is the reproducing kernel Hilbert space associated to $\k$; $\pscal{\cdot}{\cdot}_{\k}:=\pscal{\cdot}{\cdot}_{\Hk}$, $\left\|\cdot\right\|_{\k}:=\left\|\cdot\right\|_{\Hk}$.\footnote{$\Hk$ is separable
by the separability of $\X$ and the continuity of $\k$ \citep[Lemma~4.33]{steinwart08support}. These assumptions on $\X$ and $\k$ are assumed to hold throughout the paper.} The reproducing property of the kernel means that evaluation of functions in
$\Hk$ can be represented by inner products $f(x) = \pscal{f}{\k(\cdot,x)}_{\k}$ for all $x\in \X$, $f\in \Hk$.
The mean embedding of a probability measure $\P$ is defined as
\begin{align}
 \mu_\P = \int_\X \k(\cdot,x)\d \P(x)\in\Hk, \label{eq:ME}
\end{align}
where the integral is meant in Bochner sense; $\mu_\P$ exists iff
$\int_\X \left\|\k(\cdot,x)\right\|_{\k} \d \P(x) = \int_\X \sqrt{\k(x,x)} \d \P(x) <\infty$.
It is well-known that the mean embedding has mean-reproducing property
$ \P f :=\E_{x\sim\P} f(x) = \pscal{f}{\mu_\P}_{\k}$ for all $f\in \Hk$,
and it is the unique solution of the problem:
\begin{align}
\label{eq:first-characterization-mean}
\mu_\P &= \argmin_{f\in \Hk} \int_\X  \norm{f - \k(\cdot,x)}_{\k}^2 \d \P(x) \, .
\end{align}
The solution of this task can be obtained by solving the following minmax optimization
\begin{equation}
\label{eq:second-characterization-mean}
\mu_\P =  \argmin_{f\in \Hk} \sup_{g \in \Hk} J(f,g),  
\end{equation}with
$J(f,g)  = \E_{x\sim \P}
\left[ \norm{f - \k(\cdot,x)}_{\k}^2 - \norm{g - \k(\cdot,x)}_{\k}^2 \right]$.
The equivalence of \eqref{eq:first-characterization-mean} and \eqref{eq:second-characterization-mean} is obvious
since the expectation is linear. Nevertheless, this equivalence is essential in the construction of our estimators because we will below replace the expectation by a non-linear estimator of this quantity. More precisely, the unknown expectations are computed by using
the Median-of-meaN estimator (MON). Given a partition of the dataset into blocks, the MON estimator is the median of the empirical means over each block. MON estimators are naturally robust thanks to the median step.

More precisely, the procedure goes as follows. For any map $h: \X  \rightarrow \R$ and any non-empty subset $S\subseteq [N]$, denote by
$\P_S := |S|^{-1} \sum_{i\in S}\delta_{x_i}$ the empirical measure associated to the
subset $x_S$ and  $\P_S h = |S|^{-1} \sum_{i\in S}h(\xi_i)$; we will use the shorthand $\mu_{S}:= \mu_{\P_S}$.
Assume that $N\in\Z^+$ is divisible by $Q\in\Z^+$ and let $(S_q)_{q\in [Q]}$ denote a partition of $[N]$ into subsets with the same cardinality $|S_q| = N/Q$ ($\forall q \in [Q]$). The Median Of meaN (MON) is defined as
\begin{align*}
    \MOM{Q}{h}\hspace{-0.05cm} = \hspace{-0.05cm} \medi{q\in [Q]}\set{\P_{S_q}h} \hspace{-0.05cm} = \hspace{-0.05cm} \medi{q \in [Q]}\set{\pscal{h}{\mu_{S_q}}_{\k}},
\end{align*}
where assuming that $h\in \Hk$ the second equality is a consequence of the mean-reproducing property of $\mu_\P$. Specifically, in case of $Q = 1$ the MON operation reduces to the classical mean: 
$\MOM{1}{h} = N^{-1} \sum_{n=1}^N h(x_n)$.

We define the minmax MON-based estimator associated to kernel $K$ (MONK) as
\begin{align}
    \hat{\mu}_{\P,Q} &= \hat{\mu}_{\P,Q}(x_{1:N})\in\argmin_{f\in \Hk}\sup_{g\in \Hk}\tilde{J}(f,g),\label{eq:MOMK}
\end{align}where for all $f,g\in\Hk$
\begin{eqnarray*}
\lefteqn{\tilde{J}(f,g) =}\\
&&\hspace{-0.2cm}=\MOM{Q}{x \mapsto \norm{f-\k(\cdot,x)}_{\k}^2-\norm{g-\k(\cdot,x)}_{\k}^2}.
\end{eqnarray*}
When $Q=1$, since $\MOM{1}{h}$ is the empirical mean, we obtain the classical empirical mean based estimator: 
$\hat{\mu}_{\P,1}= \frac{1}{N}\sum_{n=1}^N \k(\cdot,x_n)$.

One can use the mean embedding \eqref{eq:ME} to get a semi-metric on probability measures: the
maximum mean discrepancy (MMD) of $\P$ and $\Q$ is
\begin{align*}
 \text{MMD}(\P,\Q) &:= \left\|  \mu_\P - \mu_\Q \right\|_{\k}=\sup_{f\in B_\k}\pscal{f}{\mu_\P-\mu_\Q}_{\k},
\end{align*}
where $B_\k = \{f\in \Hk: \left\|f\right\|_{\k} \le 1\}$ is the closed unit ball around the origin in $\Hk$. The second equality shows
that MMD is a specific integral probability metric \cite{muller97integral,zolotarev83probability}. Assume that we have access to $x_{1:N} \sim \P$, $y_{1:N}\sim \Q$ samples, where we assumed
the size of the two samples  to be the same for simplicity.
Denote by $\P_{S,x} := \frac{1}{|S|} \sum_{i\in S}\delta_{x_i}$ the empirical measure associated to the subset $x_S$ ($\P_{S,y}$ is defined similarly for $y$),
$\mu_{S_q,\P} := \mu_{\P_{S_q,x}}$, $\mu_{S_q,\Q} := \mu_{\P_{S_q,y}}$. We propose the following MON-based MMD
estimator
      \begin{align}\hspace*{-0.31cm}
    \MMDQhat(\P,\Q) \hspace*{-0.05cm} &=  \hspace*{-0.1cm} \sup_{f\in B_\k} \med{q\in[Q]}\left\{\pscal{f}{\mu_{S_q,\P}-\mu_{S_q,\Q}}_{\k}\right\}.\label{eq:MMD-MONK-estimator}
      \end{align}
Again, with the $Q=1$ choice, the classical V-statistic based MMD estimator \cite{gretton12kernel}  is
recovered:
 \begin{align}
  &\MMDhat(\P,\Q) = \sup_{f\in B_\k} \left[\frac{1}{N}\sum_{n \in [N]}f(x_n) - \frac{1}{N}\sum_{n \in [N]} f(y_n) \right] \nonumber\\
   &= \sqrt{\frac{1}{N^2}\sum_{i,j\in [N]} \left( K^x_{ij} +  K^y_{ij}  - 2 K^{xy}_{ij} \right)}, \label{eq:MMD-classical-estimator}
 \end{align}
where  $K^x_{ij} = \k(x_i,x_j), K^y_{ij} = \k(y_i,y_j)$ and $K^{xy}_{ij} = \k(x_i,y_j)$ for all $i,j\in[N]$. 
Changing in Eq.~\eqref{eq:MMD-classical-estimator} $\sum_{i,j\in [N]}$ to $\sum_{i,j\in [N], i \ne j}$ in case of the $K_{ij}^{x}$ and $K_{ij}^{y}$ terms gives the 
(unbiased) U-statistic based MMD estimator 
\begin{align}
  \frac{1}{N(N-1)} \sum_{\substack{i,j\in [N]\\ i\ne j}}\left(K_{ij}^x + K_{ij}^y\right) - \frac{2}{N^2}\sum_{i,j\in [N]} K_{ij}^{xy}. \label{eq:MMD-UStat}
\end{align}

Our \tb{goal} is to lay down the theoretical foundations of the $\hat{\mu}_{\P,Q}$ and  $\MMDQhat(\P,\Q)$ MONK estimators: study their finite-sample behaviour (prove optimal sub-Gaussian deviation bounds) and
establish their outlier-robustness properties.

A \tb{few additional notations} will be needed throughout the paper. $S_1\backslash S_2$ is the difference of set $S_1$ and $S_2$.  For any linear operator $A: \Hk \rightarrow \Hk$, denote by $\left\|A\right\| := \sup_{0\ne f \in \Hk} \left\|Af\right\|_{\k}/\left\|f\right\|_{\k}$ the operator norm of $A$. Let $\L(\Hk)= \left\{A: \Hk \rightarrow \Hk\text{ linear operator}: \left\|A\right\| <\infty\right\}$ be the space of bounded linear operators.
For any $A\in \L(\Hk)$, let $A^*\in \L(\Hk)$ denote the adjoint of $A$, that is the operator such that $\pscal{Af}{g}_{\k} = \pscal{f}{A^*g}_{\k}$ for all $f,g\in \Hk$.
An operator $A \in \L(\Hk)$ is called non-negative if $\pscal{Af}{f}_{\k}\ge 0$ for all $f\in \Hk$. By the separability of $\Hk$, there exists a countable orthonormal basis (ONB) $(e_i)_{i\in I}$ in $\Hk$.
$A \in \L(\Hk)$ is called  trace-class if $\left\|A \right\|_1:=\sum_{i\in I} \pscal{\left(A^*A \right)^{1/2}e_i}{e_i}_{\k}<\infty$ and in this case
$\Tr(A) := \sum_{i\in I} \pscal{Ae_i}{e_i}_{\k} < \infty$. If $A$ is non-negative and self-adjoint, then $A$ is trace class iff $\Tr(A) < \infty$; this will hold for the covariance operator ($\Sigma_{\P}$, see Eq.~\eqref{eq:covop-def}).
$A \in \L(\Hk)$ is called  Hilbert-Schmidt if $\left\|A\right\|_2^2 := \Tr\left(A^*A\right) = \sum_{i \in I} \pscal{Ae_i}{Ae_i}_{\k} < \infty$.
One can show that the definitions of trace-class and Hilbert-Schmidt operators are independent of the particular choice of the ONB $(e_i)_{i\in I}$.
Denote by $\L_1(\Hk) := \left\{A\in \L(\Hk): \left\|A\right\|_1 < \infty \right\}$ and $\L_2(\Hk) := \left\{A\in \L(\Hk): \left\|A\right\|_2 < \infty \right\}$ the class of trace-class and (Hilbert) space of Hilbert-Schmidt operators on $\Hk$, respectively.
The tensor product of $a,b\in\Hk$ is
$(a\otimes b)(c) = a  \pscal{b}{c}_{\k}, \quad (\forall c \in \Hk)$, $a\otimes b \in \L_2(\Hk)$ and $\left\|a\otimes b\right\|_2 = \left\|a\right\|_{\k} \left\|b\right\|_{\k}$.
$\L_2(\Hk)\cong  \Hk \otimes \Hk$ where the r.h.s.\ denotes the tensor product of Hilbert spaces defined as the closure of  $\big\{\sum_{i=1}^n a_i \otimes b_i: a_i,b_i\in\Hk\, (i\in[n]),\, n\in \Z^+\big\}$. Whenever $\int_{\X}  \left\|\k(\cdot,x)  \otimes \k(\cdot,x) \right\|_2 \d \P(x) = \int_\X \k(x,x) \d \P(x) < \infty$,
let $\Sigma_{\P}$ denote the covariance operator
\begin{align}
\Sigma_{\P} &= \E_{x\sim \P} \left( \left[\k(\cdot,x) - \mu_\P\right] \otimes \left[\k(\cdot,x) - \mu_\P\right] \right) \in \L_2(\Hk), \label{eq:covop-def}
\end{align}
where the expectation (integral) is again meant in Bochner sense. $\Sigma_\P$ is non-negative, self-adjoint, moreover it has covariance-reproducing property
$\pscal{f}{\Sigma_{\P}f}_{\k} = \E_{x\sim \P}[f(x) - \P f]^2$. It is known that $\left\|A\right\| \le \left\|A\right\|_2 \le \left\|A\right\|_1$.

\section{Main Results}\label{sec:results}
Below we present our main results on the MONK estimators, followed by a discussion.
We allow that  $N_c$ elements($(x_{n_j})_{j=1}^{N_c}$ ) of the sample $x_{1:N}$ are arbitrarily corrupted (In MMD
estimation $\{(x_{n_j},y_{n_j})\}_{j=1}^{N_c}$ can be contaminated). The
number of corrupted samples can be (almost) half of the number of blocks, in other words, there exists $\delta  \in \left(0,1/2\right]$ such that $N_c  \le Q(1/2-\delta)$. If the data are free from 
contaminations, then $N_c=0$ and $\delta=1/2$. Using these notations, we can prove the following optimal sub-Gaussian deviation bounds on the MONK estimators.
\begin{theorem}[Consistency \& outlier-robustness of $\hat{\mu}_{\P,Q}$]\label{thm:RB}
 Assume that $\Sigma_{\P} \in \L_1(\Hk)$.
  Then,  for any $\eta \in (0,1)$ such that $Q=72\delta^{-2} \ln\left(1/\eta\right)$ satisfies $Q\in (N_c/(1/2-\delta),N/2)$, with probability at least $1-\eta$,
  \begin{align*}
  &\left\|\hat{\mu}_{\P,Q} - \mu_\P\right\|_{\k} \\
  &\le \frac{12 \left(1+\sqrt{2}\right)}{\delta}\max\left(\sqrt{\frac{6\left\|\Sigma_{\P} \right\| \ln\left(1/\eta\right)}{\delta N}},2\sqrt{\frac{\Tr{(\Sigma_\P)}}{N}}\right).
      \end{align*}
\end{theorem}

\begin{theorem}[Consistency \& outlier-robustness of $\MMDQhat(\P,\Q)$]\label{thm:RB2}
  Assume that $\Sigma_{\P}$ and $\Sigma_{\Q} \in \L_1(\Hk)$.
 Then,  for any $\eta \in (0,1)$ such that $Q=72\delta^{-2} \ln\left(1/\eta\right)$ satisfies $Q\in (N_c/(1/2-\delta),N/2)$, with probability at least $1-\eta$,
          \begin{align*}
            &\left|\MMDQhat(\P,\Q) - \MMD(\P,\Q) \right|\\
    &\le \frac{12 \max\left(\sqrt{\frac{\left( \left\|\Sigma_{\P} \right\|  + \left\|\Sigma_{\Q} \right\|\right) \ln\left(1/\eta\right)}{\delta N}},2\sqrt{\frac{\Tr{(\Sigma_\P)} + \Tr{(\Sigma_\Q)}}{N}}\right)}{\delta}.
        \end{align*}
\end{theorem}
\begin{proof}[Proof (sketch)] The technical challenge is to get the optimal deviation bounds under the  (mild) trace-class assumption. The reasonings for the mean embedding and MMD follow a
similar high-level idea; here we focus on the former. First we show that the analysis can be reduced to the unit ball in $\Hk$ by proving that 
$\left\|\hat{\mu}_{\P,Q} - \mu_\P\right\|_{\k} \le (1+\sqrt{2})r_{Q,N}$, 
where  
$r_{Q,N}  = \sup_{f\in B_\k}\text{MON}_Q \big[x \mapsto \psh{f}{\k(\cdot,x)-\mu_\P}_{\k}\big]
 = \sup_{f\in B_\k} \med{q \in [Q]}\{r(f,q)\}$
with $r(f,q) = \pscal{f}{\mu_{S_q}-\mu_\P}_{\k}$.
The Chebyshev inequality with a Lipschitz argument allows us to control the probability of the event $\{r_{Q,N} \le \epsilon\}$ using the variable $Z = \sup_{f\in B_\k} \sum_{q\in U} \left[\phi\left(2r(f,q)/\epsilon\right) -  \E \phi\left(2r(f,q)/\epsilon\right)\right]$,
where $U$ stands for the indices of the uncorrupted blocks and $\phi(t) = (t-1) \I_{1\le t \le 2} + \I_{t\ge 2}$. The bounded difference property of the $Z$ supremum of empirical processes
guarantees its concentration around the expectation by using the McDiarmid inequality. The symmetrization technique combined with the Talagrand's contraction principle of Rademacher processes (thanks to
the Lipschitz property of $\phi$), followed by an other symmetrization leads to the deviation bound. Details are provided in Section~\ref{sec:proof:thm1}-\ref{sec:proof:thm2} (for Theorem~\ref{thm:RB}-\ref{thm:RB2}) in the supplementary material.
\end{proof}

\tb{Remarks}:
\begin{itemize}[labelindent=0cm,leftmargin=*,topsep=0cm,partopsep=0cm,parsep=0cm,itemsep=0cm]
  \item Dependence on $N$: These finite-sample guarantees show that the MONK estimators
      \begin{itemize}[labelindent=0cm,leftmargin=*,topsep=0cm,partopsep=0cm,parsep=0cm,itemsep=0cm]
    \item have optimal $N^{-1/2}$-rate---by recalling \citet{tolstikhin16minimax,tolstikhin17minimax}'s discussed results---, and
    \item they are robust to outliers, providing consistent estimates with high probability even under arbitrary adversarial contamination (affecting less than half of the samples).
      \end{itemize}\vspace{0.5mm}
  \item Dependence on $\delta$:  Recall that larger $\delta$ corresponds to less outliers, i.e., cleaner data in which case the bounds above become tighter. In other words,
    making use of medians the MONK estimators show robustness to outliers; this property is a nice byproduct of our optimal sub-Gaussian deviation bound. Whether
    this robustness to outliers is optimal in the studied setting is an open question.\vspace{0.5mm}
  \item Dependence on $\Sigma$:  It is worth contrasting the rates obtained in Theorem~\ref{thm:RB} and that of the tournament procedures \cite{LugosiMendelson2017-2} derived for the finite-dimensional case.
  The latter paper elegantly resolved a long-lasting open question concerning the optimal dependency in terms of $\Sigma$.
  Theorem~\ref{thm:RB} proves the same dependency in the infinite-dimensional case, while giving rise to computionally tractable algorithms (Section~\ref{sec:MONK-implementation}).\vspace{0.5mm}
  \item Separation rate: Theorem~\ref{thm:RB2} also shows that fixing the trace of the covariance operators of $\P$ and $\Q$, the MON-based MMD estimator can separate $\P$ and $\Q$ at the rate of $N^{-1/2}$.
  \item Breakdown point:  Our finite-sample bounds imply that the proposed MONK estimators using $Q$ blocks is resistant to $Q/2$ outliers. Since $Q$ is allowed to grow with $N$ (it can be
  can be chosen to be almost $N/2$), this specifically means that the breakdown point of our estimators can be $25\%$.
\end{itemize}

\section{Computing the MONK Estimator} \label{sec:MONK-implementation}
This section is dedicated to the computation\footnote{The Python code reproducing our numerical experiments is available at \url{https://bitbucket.org/TimotheeMathieu/monk-mmd}; it relies on the ITE toolbox \cite{szabo14information}.} of the analyzed MONK estimators; particularly we will focus on the MMD estimator given in Eq.~\eqref{eq:MMD-MONK-estimator}.
Numerical illustrations are provided in Section~\ref{sec:numerical-demos}.
Recall that the MONK estimator for MMD [Eq.~\eqref{eq:MMD-MONK-estimator}] is given by
\begin{align}
\label{eq:MMDQ-opt2}
  &\MMDQhat(\P,\Q)  \\
  & =  \sup_{f\in B_\k} \med{q\in[Q]} \left\{ \frac{1}{|S_q|}\sum_{j\in S_q}f(x_j)  - \frac{1}{|S_q|}\sum_{j\in S_q}f(y_j) \right\}.\nonumber
\end{align}
By the representer theorem \cite{scholkopf01generalized}, the optimal $f$ can be expressed as
 \begin{align}
f(\b a, \b b) = \sum_{n\in [N]}a_n \k(\cdot,x_n) + \sum_{n\in [N]}b_n\k(\cdot,y_n), \label{eq:f}
\end{align}
where $\b a = (a_n)_{n\in [N]}\in \R^N$ and $\b b = (b_n)_{n\in [N]}\in \R^N$.
Denote $\b c= [\b a;\b b] \in \R^{2N}$, $\b \k = [\b \k_{xx}, \b \k_{xy}; \b \k_{yx}, \b \k_{yy}]\in \R^{2N\times 2N}$, $\b \k_{xx} = [\k(x_i,x_j)]_{i,j\in [N]}\in \R^{N\times N}$, $\b \k_{xy} = [\k(x_i,y_j)]_{i,j\in [N]} = \b \k_{yx}^* \in \R^{N\times N}$, $\b \k_{yy} = [\k(y_i,y_j)]_{i,j\in [N]}\in \R^{N\times N}$.   With these notations, the
optimisation problem \eqref{eq:MMDQ-opt2} can be rewritten as
  \begin{align}
    \max_{\b c\in \R^{2N}: \b c^* \b \k \b c \le 1} \med{q\in[Q]} \Big\{ |S_q|^{-1}[\b 1_q; -\b 1_q]^* \b \k \b c \Big\}, \label{eq:c0}
  \end{align}
  where $\b 1_q \in \R^N$ is indicator vector of the block $S_q$. To enable efficient optimization we follow a block-coordinate descent (BCD)-type scheme: choose the $q_{\text{m}}\in [N]$ index
  for which the median is attained in \eqref{eq:c0}, and solve
  \begin{align}
    \max_{\b c\in \R^{2N}: \b c^* \b \k \b c \le 1} |S_{q_{\text{m}}}|^{-1} [\b 1_{q_{\text{m}}}; -\b 1_{q_{\text{m}}}]^* \b \k \b c. \label{eq:c}
  \end{align}
   This optimization problem can be solved analytically:
$\b c = \frac{ [\b 1_{q_{\text{m}}}; -\b 1_{q_{\text{m}}}]}{\left\| \b L^* [\b 1_{q_{\text{m}}}; -\b 1_{q_{\text{m}}}] \right\|_2}$,
where $\b L$ is the Cholesky factor of $\b \k$ ($\b \k = \b L \b L^*$). The observations are shuffled after each iteration.
The pseudo-code of the final MONK BCD estimator is summarized in Algorithm~\ref{alg:MONK-BCD}. 

Notice that computing $\b L$ in MONK BCD costs $O(N^3)$, which can be prohibitive for large sample size. 
In order to alleviate this bottleneck we also consider an approximate version of MONK BCD (referred to as MONK BCD-Fast), where the $\sum_{n\in [N]}$ summation after plugging \eqref{eq:f} into 
\eqref{eq:MMDQ-opt2} is replaced with $\sum_{n\in S_q}$:
\begin{eqnarray*}
    \lefteqn{\max_{\b c = [\b a, \b b]\in \R^{2N}: \b c^* \b \k \b c \le 1} \med{q\in[Q]} \left\{ \frac{1}{|S_q|}\sum_{j,n\in S_q}\left[a_n K(x_j,x_n) + \right. \right.} \\
     &&\hspace*{-0.7cm}\left. \left. b_n K(x_j,y_n)\right]  - 
     \frac{\sum_{j,n\in S_q}\left[a_n K(y_j,x_n) + b_n K(y_j,y_n)\right]}{|S_q|} \right\}.
\end{eqnarray*}
This modification allows local computations restricted to blocks and improved running time. 
The samples are shuffled periodically (e.g., at every $10$th iterations) to renew the blocks. The resulting method is presented in Algorithm~\ref{alg:MONK-BCD-Fast}. 
The computational complexity of the different MMD estimators are summarized in Table~\ref{tab:complexity}.

\begin{algorithm}[t] 
   \caption{MONK BCD estimator for MMD}
   \label{alg:MONK-BCD}
    \begin{algorithmic}
      \STATE \tb{Input:} Aggregated Gram matrix: $\b \k$ with Cholesky factor $\b L$ ($\b \k = \b L \b L^*$).
      \FORALL{$t=1,\ldots,T$}
  \STATE Generate a random permutation of $[N]$: $\sigma$.
  \STATE Shuffle the samples according to $\sigma$: for $\forall q\in [Q]$
      \vspace*{-0.2cm}
      \begin{align*}
	  S_q = \left\{\sigma\left((q-1)\frac{N}{Q}+1\right),\ldots,\sigma\left(q\frac{N}{Q}\right)\right\}.
      \end{align*}
      \vspace*{-0.4cm}
  \STATE Find the block attaining the median ($q_{\text{m}}$): 
    \vspace*{-0.2cm}
    \begin{align*}
	\frac{[\b 1_{q_{\text{m}}}; -\b 1_{q_{\text{m}}}]^* \b \k \b c}{|S_{q_{\text{m}}}|} = \med{q\in[Q]}  \frac{[\b 1_q; -\b 1_q]^* \b \k \b c}{|S_q|}.
    \end{align*}
    \vspace*{-0.4cm}
  \STATE Compute the coefficient vector: $\b c = \frac{ \left[\b 1_{q_{\text{m}}}; -\b 1_{q_{\text{m}}}\right]}{\left\| \b L^* [\b 1_{q_{\text{m}}}; -\b 1_{q_{\text{m}}}] \right\|_2}$.
      \ENDFOR
      \STATE \tb{Output:} $\med{q\in[Q]} \Big( \frac{1}{|S_q|}[\b 1_q; -\b 1_q]^* \b \k \b c \Big)$
    \end{algorithmic}
\end{algorithm}

\begin{algorithm} 
\begin{minipage}{\textwidth/2}
   \caption{MONK BCD-Fast estimator for MMD}
   \label{alg:MONK-BCD-Fast}
    \begin{algorithmic}
    \STATE \tb{Input:} Aggregated Gram matrix: $\b \k$ with Cholesky factor $\b L$ ($\b \k = \b L \b L^*$). Incides at which we shuffle: $J$.
      \FORALL{$t=1,\ldots,T$}
      \IF{$t\in J$}
		\STATE Generate a random permutation of $[N]$: $\sigma$.
		\STATE Shuffle the samples according to $\sigma$: for $\forall q\in [Q]$
		\vspace*{-0.2cm}
		\begin{align*}
		    S_q = \left\{\sigma\left((q-1)\frac{N}{Q}+1\right),\ldots,\sigma\left(q\frac{N}{Q}\right)\right\}.
		\end{align*}
		\vspace*{-0.4cm}
		\STATE Compute the Gram matrices and the Cholesky factors on each block $\b \k_q$ and $\b L_q$ for $q\in[Q]$. 
	  \ENDIF
  \STATE Find the block\footnote{$\mathds{1}_q\in \R^{|S_q|}$ denotes the vector of ones of size $|S_q|$. \label{Algs:shared-footnote}} attaining the median ($q_{\text{m}}$): 
  		\vspace*{-0.2cm}
		\begin{align*}
		      \frac{[\mathds{1}_{q_{\text{m}}}; -\mathds{1}_{q_{\text{m}}}]^* \b \k_{q_{\text{m}}} \b c_{q_{\text{m}}}}{|S_{q_{\text{m}}}|} = \med{q\in[Q]}  \frac{[\mathds{1}_q; -\mathds{1}_q]^* \b \k_{q} \b c_q}{|S_q|}.		 
		\end{align*}
		\vspace*{-0.4cm}
  \STATE Update the coefficient vector: $\b c_{q_{\text{m}}} = \frac{ \left[\mathds{1}_{q_{\text{m}}}; -\mathds{1}_{q_{\text{m}}}\right]}{\left\| \b L_{q_{\text{m}}}^* [\mathds{1}_{q_{\text{m}}}; -\mathds{1}_{q_{\text{m}}}] \right\|_2}$.
      \ENDFOR
      \STATE \tb{Output:} $\med{q\in[Q]} \Big( \frac{1}{|S_q|}[\mathds{1}_q; -\mathds{1}_q]^* \b \k_q \b c_q \Big)$
    \end{algorithmic}
\end{minipage}    
\end{algorithm}

\begin{table}
 \caption{Computational complexity of MMD estimators. $N$: sample number, $Q$: number of blocks, $T$: number of iterations.} \label{tab:complexity}
 \begin{center}
     \begin{tabular}{ll}
      \toprule
      Method & Complexity \\\toprule
       U-Stat & $\O\left(N^2\right)$ \\
       MONK BCD & $\O\left( N^3 + T \left[N^2 + Q\log(Q) \right] \right)$ \\
       MONK BCD-Fast & $\O\left( \frac{N^3}{Q^2} + T\left[\frac{N^2}{Q} + Q \log(Q) \right]\right)$ \\\bottomrule
     \end{tabular}
 \end{center}
\end{table}

\section{Numerical Illustrations} \label{sec:numerical-demos}
In this section, we demonstrate the performance of the proposed MONK estimators.  We exemplify the idea on the MMD estimator [Eq.~\eqref{eq:MMD-MONK-estimator}] with the BCD optimization schemes (MONK BCD and 
MONK BCD-Fast) discussed in Section~\ref{sec:MONK-implementation}. Our baseline is the classical U-statistic based MMD estimator [Eq.~\eqref{eq:MMD-UStat}; referred to as U-Stat in the sequel].

The primary goal in the first set of experiments is to understand and demonstrate various aspects of the estimators for $(\k,\P,\Q)$ triplets \citep[Table~3.3]{maundet17kernel}
when analytical expression is available for MMD. This is the case for polynomial and RBF kernels ($K$), with Gaussian distributions ($\P$, $\Q$). Notice that in the first (second) case the features 
are unbounded (bounded). Our second numerical example illustrates the applicability of the studied MONK estimators in biological context, in discriminating DNA subsequences with string kernel. 

\tb{Experiment-1:} We used the quadratic and the RBF kernel with bandwith $\sigma=1$ for demonstration purposes and investigated the estimation error compared to the true MMD value: $|\MMDQhat(\P,\Q)-\MMD(\P,\Q)|$. The errors are aggregates over $100$ Monte-Carlo simulations, summarized in the median and quartile values.
The number of samples ($N$) was chosen from $\{200, 400, \ldots, 2000\}$.

We considered three different experimental settings for  $(\P,\Q)$ and the absence/presence of outliers:
\begin{enumerate}[labelindent=0cm,leftmargin=*,topsep=0cm,partopsep=0cm,parsep=0cm,itemsep=0cm]
  \item Gaussian distributions with no outliers: In this case $\P = \mathcal{N}\left(\mu_1,\sigma_1^2\right)$ and $\Q = \mathcal{N}\left(\mu_2,\sigma_2^2\right)$ were normal where
	$(\mu_1,\sigma_1) \ne (\mu_2,\sigma_2)$, $\mu_1$, $\sigma_1$, $\mu_2$,
	$\sigma_2$ were randomly chosen from the $[0,1]$ interval, and then their values were fixed. The estimators had access to
	$(x_n)_{n=1}^N \stackrel{\text{i.i.d.}}{\sim} \P$ and $(y_n)_{n=1}^N \stackrel{\text{i.i.d.}}{\sim} \Q$.\vspace{2mm}
  \item Gaussian distributions with outliers: This setting is a corrupted version of the first one. Particularly, the dataset consisted of $(x_n)_{n=1}^{N-5} \stackrel{\text{i.i.d.}}{\sim} \P$,
	$(y_n)_{n=1}^{N-5} \stackrel{\text{i.i.d.}}{\sim} \Q$, while the remaining $5$-$5$ samples were set to $x_{N-4} = \ldots = x_N = 2000$,
	$y_{N-4}=\dots =y_N=4000$. \vspace{2mm}
  \item Pareto distribution without outliers: In this case $\P = \Q = Pareto(3)$ 
	hence $\MMD(\P,\Q) = 0$ and the estimators used $(x_n)_{n=1}^N \stackrel{\text{i.i.d.}}{\sim} \P$ and $(y_n)_{n=1}^N \stackrel{\text{i.i.d.}}{\sim} \Q$.
\end{enumerate}
The $3$ experiments were constructed to understand different aspects of the estimators: how a few outliers can ruin classical estimators (as we move from Experiment-1 to Experiment-2); in Experiment-3
 the heavyness of the tail of a Pareto distribution makes the task non-trivial. 

Our results on the three datasets with various $Q$ choices are summarized in Fig.~\ref{fig:comparison}. As we can see from Fig.~\ref{fig:poly} and Fig.~\ref{fig:rbf} in the outlier-free case, the MONK estimators are slower
than the U-statistic based one; the accuracy is of the same order for both kernels. As demonstrated by
Fig.~\ref{fig:poly_out} in the corrupted setup even a small number of outliers can completely ruin traditional MMD estimators for unbounded features while the MONK
estimators are naturally robust to outliers with suitable choice of $Q$;\footnote{In 
case of unknown $N_c$, one could choose $Q$ adaptively by the Lepski method (see for example \cite{MR3576558}) at the price of increasing the computational effort. 
Though the resulting $Q$ would increase the computational time, it would be adaptive thanks to its data-driven nature, and would benefit from the same guarantee as the fixed $Q$ appearing in Theorem~\ref{thm:RB}-\ref{thm:RB2}.} \label{footnote:Lepski} this is precisely the setting the MONK estimators were designed for.
In case of bounded kernels (Fig.~\ref{fig:rbf_out}), by construction, traditional MMD estimators are resistant to outliers; the MONK BCD-Fast method achieves comparable performance.
In the final Pareto experiment (Fig.~\ref{fig:poly_pareto} and Fig.~\ref{fig:rbf_pareto}) where the distribution produces ``natural outliers'',  again MONK estimators are more robust with respect to corruption than the one relying on U-statistics in the case of polynomial kernel.
These experiments illustrate the power of the studied MONK schemes: these estimators achieve comparable performance in case of bounded features, while for unbounded features they can efficiently cope with the presence of outliers.

\begin{figure*}
\begin{center}
    \subfloat[Gaussian distribution, $N_c =0$ (no outlier), quadratic kernel.\label{fig:poly}]{
      \includegraphics[width=0.32\textwidth]{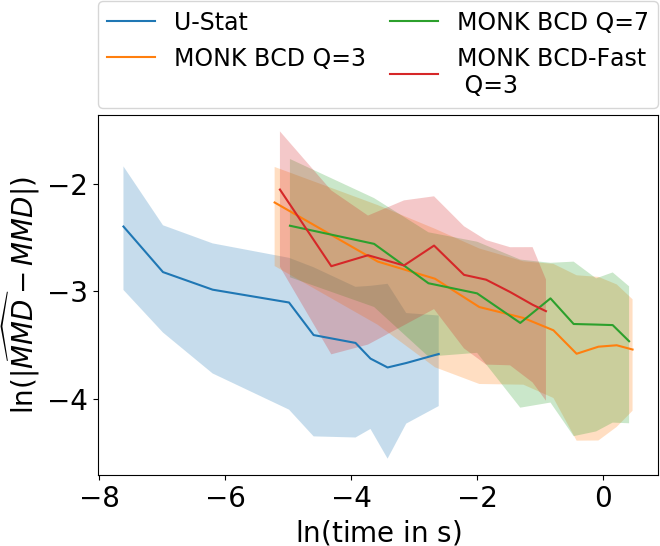}
                         }
\subfloat[Gaussian distribution,  $N_c=5$ outliers, quadratic kernel.\label{fig:poly_out}]{
      \includegraphics[width=0.32\textwidth]{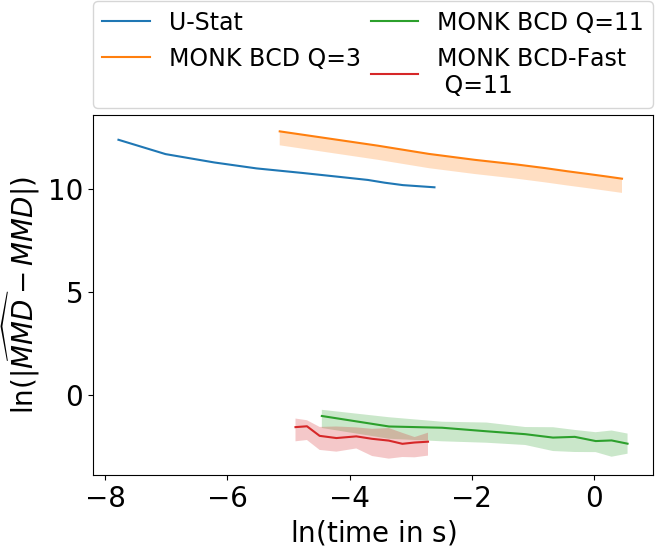}
                         }
\subfloat[Pareto distribution, quadratic kernel.\label{fig:poly_pareto}]{
      \includegraphics[width=0.32\textwidth]{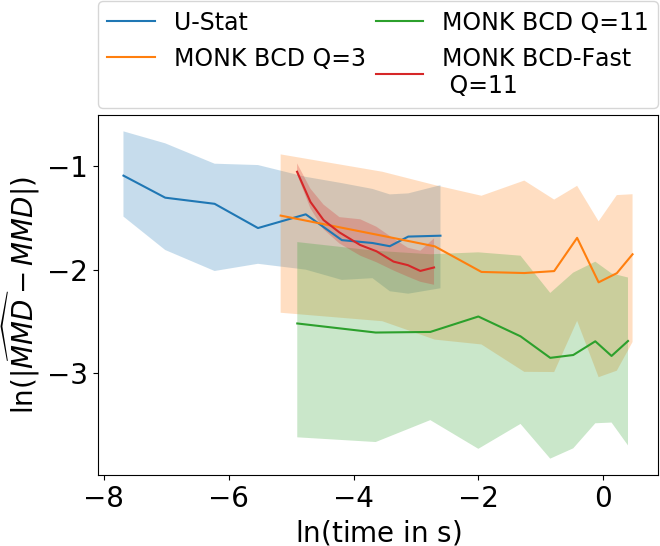}
                         }                         
  \\
    \subfloat[Gaussian distribution, $N_c=0$ (no outlier), RBF kernel.\label{fig:rbf}]{
      \includegraphics[width=0.32\textwidth]{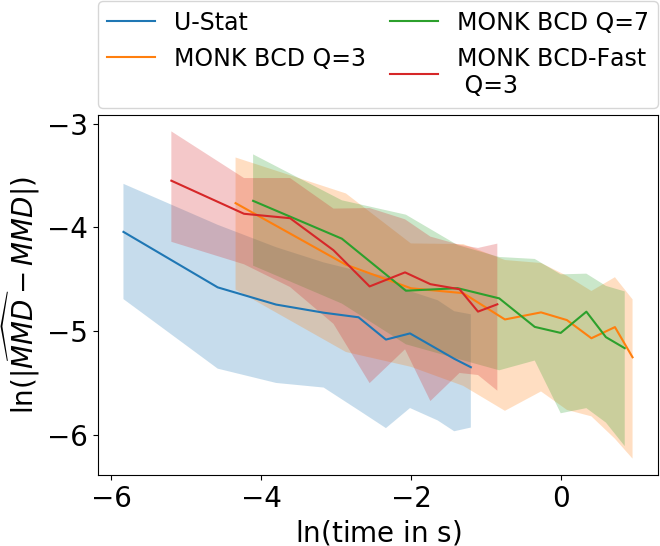}
                         }                         
\subfloat[Gaussian distribution,  $N_c=5$ outliers, RBF kernel.\label{fig:rbf_out}]{
      \includegraphics[width=0.32\textwidth]{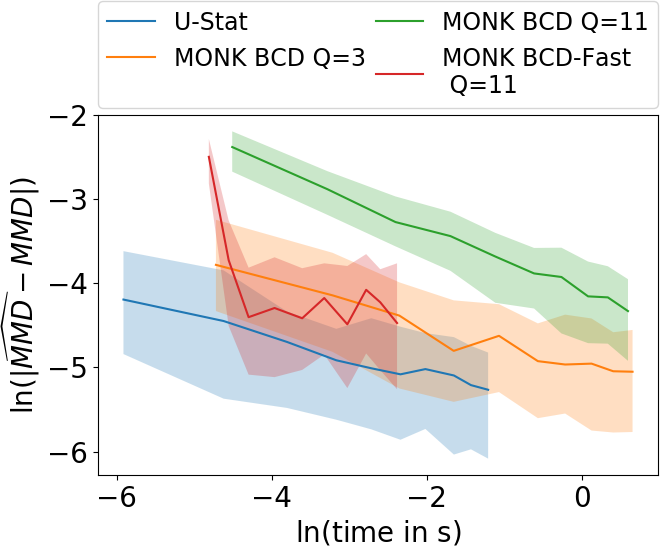}
                         }                         
\subfloat[Pareto distribution, RBF kernel.\label{fig:rbf_pareto}]{
      \includegraphics[width=0.32\textwidth]{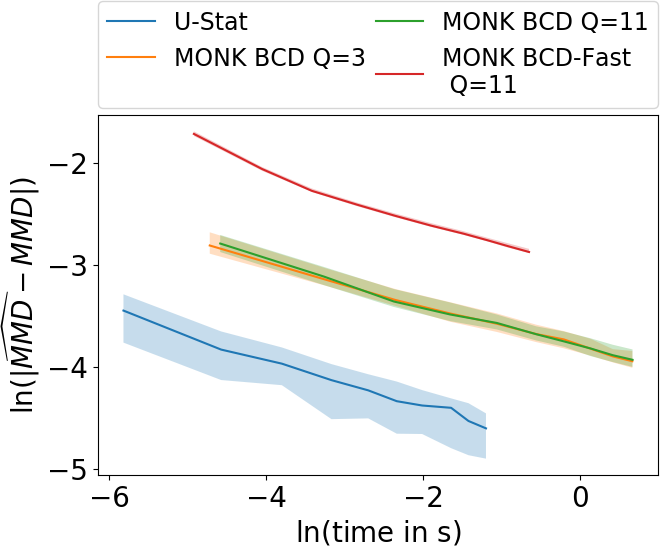}
                         }
\caption{Performance of the MMD estimators: median and quartiles of $\ln(|\MMDQhat(\P,\Q)-\MMD(\P,\Q)|)$. Columns from left to right: Experiment-1 -- Experiment-3. Top: quadratic kernel, bottom: RBF kernel.\label{fig:comparison}}
\end{center}
\end{figure*}

\tb{Experiment-2 (discrimination of DNA subsequences):} In order to demonstrate the applicability of our estimators in biological context, we chose 
a DNA benchmark from the UCI repository~\cite{Dua:2017}, the Molecular Biology (Splice-junction Gene Sequences) Data Set.
The dataset consists of $3190$ instances of $60$-character-long DNA subsequences. The problem is to 
recognize, given a sequence of DNA, the boundaries between exons (the parts of the DNA sequence retained after splicing) and 
introns (the parts of the DNA sequence that are spliced out). This task consists of two subproblems, 
identifying the exon/intron boundaries (referred to as EI sites) and the intron/exon boundaries (IE sites).\footnote{In the biological community, IE borders are referred to 
as ``acceptors'' while EI borders are referred to as ``donors''.} We took $1532$ of these samples by selecting $766$ instances from both the EI and the IE classes 
(the class of those being neither EI nor IE is more heterogeneous and thus we dumped it from the study), and investigated the discriminability of the EI and IE categories. 
We represented the DNA sequences as  strings ($\X$), chose $\k$ as the String Subsequence Kernel \citep{lodhi02text} to compute MMD, and performed two-sample testing based on MMD using the MONK BCD, MONK BCD-Fast and U-Stat estimators.
For completeness the pseudocode of the hypothesis test is detailed in Algorithm~\ref{alg:DNA} (Section~\ref{sec:experiment-2}). $Q$, the number of blocks in the MONK 
techniques, was equal to $5$. The significance level was $\alpha = 0.05$. To  assess the variability of the results $400$ Monte Carlo simulations were 
performed, each time uniformly sampling $N$ points without replacement resulting in $(X_n)_{n\in [N]}$ and $(Y_n)_{n\in [N]}$. To provide more detailed insights the aggregated values of 
$\widehat{\MMD}(\text{EI},\text{IE})-\hat{q}_{1-\alpha}$, $\widehat{\MMD}(\text{EI},\text{EI})-\hat{q}_{1-\alpha}$ and $\widehat{\MMD}(\text{IE},\text{IE})-\hat{q}_{1-\alpha}$
are summarized in Fig.~\ref{fig:dna}, where $\hat{q}_{1-\alpha}$ is the estimated $(1-\alpha)$-quantile via $B=150$ bootstrap permutations. In the ideal case, 
$\widehat{\MMD}-\hat{q}_{1-\alpha}$ is positive (negative) in the inter-class (intra-class) experiments. As Fig.~\ref{fig:dna} shows all 3 techniques are able to 
solve the task, both in the inter-class (when the null hypothesis does not hold; Fig.~\ref{fig:dna:inter}) and the intra-class experiment (null holds; Fig.~\ref{fig:dna:intra:EI} and 
Fig.~\ref{fig:dna:intra:IE}), and 
they converge to a good and stable performance as a function of the sample number. 
It is important to note that the MONK BCD-Fast method is especially well-adapted to problems where the kernel computation (such as the String Subsequence Kernel) or the sample size is a bottleneck, as its computation is 
often significantly faster compared to the U-Stat technique. For example, taking all the samples ($N=766$) in the DNA benchmark with $Q=15$, computing MONK BCD-Fast (U-Stat) takes $32s$ ($1m28s$).
These results illustrate the applicability of our estimators in gene analysis.

\begin{figure*}
\begin{center}
  \subfloat[Inter-class: EI-IE\label{fig:dna:inter}]{\includegraphics[width=0.32\textwidth]{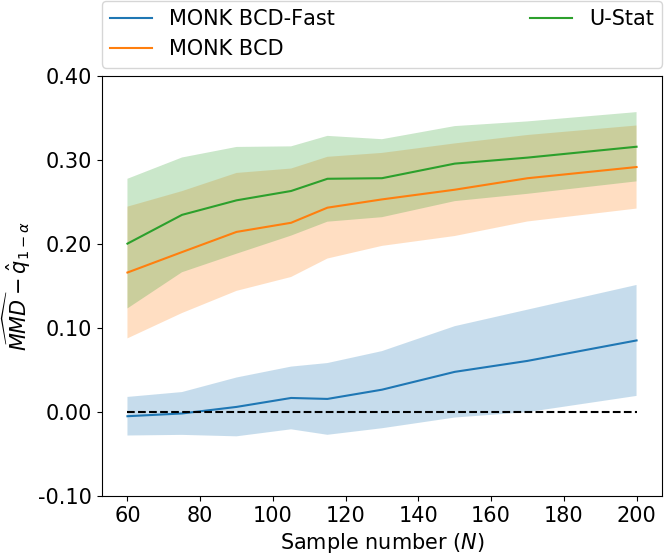}}\hspace{0.1cm}
  \subfloat[Intra-class: EI-EI\label{fig:dna:intra:EI}]{\includegraphics[width=0.32\textwidth]{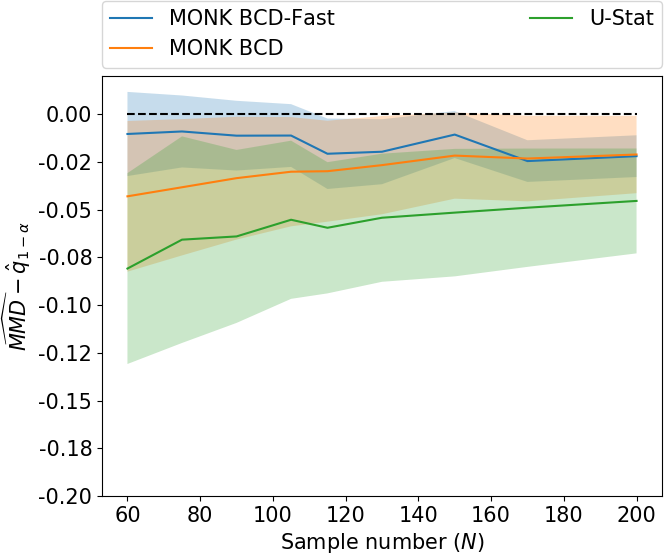}}\hspace{0.1cm}
  \subfloat[Intra-class: IE-IE\label{fig:dna:intra:IE}]{\includegraphics[width=0.32\textwidth]{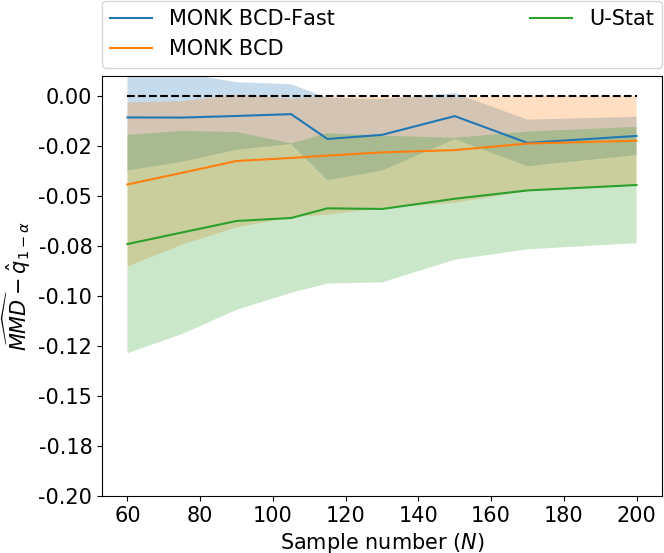}}
  \caption{Inter-class and intra-class MMD estimates as a function of the sample number compared to the bootstrap-estimated $(1-\alpha)$-quantile: $\MMDhat - \hat{q}_{1-\alpha}$; mean $\pm$ std. The null hypothesis is rejected iff 
  $\MMDhat - \hat{q}_{1-\alpha}>0$. Notice the different scale of $\MMDhat - \hat{q}_{1-\alpha}$ in the inter-class and the intra-class experiments.\label{fig:dna}}
\end{center}
\end{figure*}

\section*{Acknowledgements}
Guillaume Lecu{\'e} is supported by a grant of the French National Research Agency (ANR), “Investissements d'Avenir” (LabEx Ecodec/ANR-11-LABX-0047).

\bibliography{./BIB/MONK_short}

\begin{thebibliography}{90}
\providecommand{\natexlab}[1]{#1}
\providecommand{\url}[1]{\texttt{#1}}
\expandafter\ifx\csname urlstyle\endcsname\relax
  \providecommand{\doi}[1]{doi: #1}\else
  \providecommand{\doi}{doi: \begingroup \urlstyle{rm}\Url}\fi

\bibitem[Alam et~al.(2018)Alam, Fukumizu, and Wang]{alam18influence}
Alam, M.~A., Fukumizu, K., and Wang, Y.-P.
\newblock Influence function and robust variant of kernel canonical correlation
  analysis.
\newblock \emph{Neurocomputing}, 304:\penalty0 12--29, 2018.

\bibitem[Alon et~al.(1999)Alon, Matias, and Szegedy]{MR1688610}
Alon, N., Matias, Y., and Szegedy, M.
\newblock The space complexity of approximating the frequency moments.
\newblock \emph{Journal of Computer and System Sciences}, 58\penalty0 (1, part
  2):\penalty0 137--147, 1999.

\bibitem[Aronszajn(1950)]{aronszajn50theory}
Aronszajn, N.
\newblock Theory of reproducing kernels.
\newblock \emph{Transactions of the American Mathematical Society},
  68:\penalty0 337--404, 1950.

\bibitem[Audibert \& Catoni(2011)Audibert and Catoni]{MR2906886}
Audibert, J.-Y. and Catoni, O.
\newblock Robust linear least squares regression.
\newblock \emph{The Annals of Statistics}, 39\penalty0 (5):\penalty0
  2766--2794, 2011.

\bibitem[Balasubramanian et~al.(2017)Balasubramanian, Li, and
  Yuan]{balasubramanian17optimality}
Balasubramanian, K., Li, T., and Yuan, M.
\newblock On the optimality of kernel-embedding based goodness-of-fit tests.
\newblock Technical report, 2017.
\newblock (\url{https://arxiv.org/abs/1709.08148}).

\bibitem[Baringhaus \& Franz(2004)Baringhaus and Franz]{baringhaus04new}
Baringhaus, L. and Franz, C.
\newblock On a new multivariate two-sample test.
\newblock \emph{Journal of Multivariate Analysis}, 88:\penalty0 190--206, 2004.

\bibitem[Berlinet \& Thomas-Agnan(2004)Berlinet and
  Thomas-Agnan]{berlinet04reproducing}
Berlinet, A. and Thomas-Agnan, C.
\newblock \emph{Reproducing Kernel Hilbert Spaces in Probability and
  Statistics}.
\newblock Kluwer, 2004.

\bibitem[Binkowski et~al.(2018)Binkowski, Sutherland, Arbel, and
  Gretton]{binkowski18demistifying}
Binkowski, M., Sutherland, D.~J., Arbel, M., and Gretton, A.
\newblock Demystifying {MMD} {GAN}s.
\newblock In \emph{ICLR}, 2018.

\bibitem[Blanchard et~al.(2017)Blanchard, Deshmukh, Dogan, Lee, and
  Scott]{blanchard17domain}
Blanchard, G., Deshmukh, A.~A., Dogan, U., Lee, G., and Scott, C.
\newblock Domain generalization by marginal transfer learning.
\newblock Technical report, 2017.
\newblock (\url{https://arxiv.org/abs/1711.07910}).

\bibitem[Catoni(2012)]{MR3052407}
Catoni, O.
\newblock Challenging the empirical mean and empirical variance: a deviation
  study.
\newblock \emph{Annales de l'Institut Henri Poincar\'e Probabilit\'es et
  Statistiques}, 48\penalty0 (4):\penalty0 1148--1185, 2012.

\bibitem[Catoni \& Giulini(2017)Catoni and Giulini]{Catoni-Giullini2017}
Catoni, O. and Giulini, I.
\newblock Dimension-free {PAC}-{B}ayesian bounds for matrices, vectors, and
  linear least squares regression.
\newblock Technical report, 2017.
\newblock (\url{https://arxiv.org/abs/1712.02747}).

\bibitem[Cherapanamjeri et~al.(2019)Cherapanamjeri, Flammarion, and
  Bartlett]{cherapanamjeri19fast}
Cherapanamjeri, Y., Flammarion, N., and Bartlett, P.~L.
\newblock Fast mean estimation with sub-{G}aussian rates.
\newblock Technical report, 2019.
\newblock (\url{https://arxiv.org/abs/1902.01998}).

\bibitem[Collins \& Duffy(2001)Collins and Duffy]{collins01convolution}
Collins, M. and Duffy, N.
\newblock Convolution kernels for natural language.
\newblock In \emph{NIPS}, pp.\  625--632, 2001.

\bibitem[Cuturi(2011)]{cuturi11fast}
Cuturi, M.
\newblock Fast global alignment kernels.
\newblock In \emph{ICML}, pp.\  929--936, 2011.

\bibitem[Cuturi et~al.(2005)Cuturi, Fukumizu, and Vert]{cuturi05semigroup}
Cuturi, M., Fukumizu, K., and Vert, J.-P.
\newblock Semigroup kernels on measures.
\newblock \emph{Journal of Machine Learning Research}, 6:\penalty0 1169--1198,
  2005.

\bibitem[Devroye et~al.(2016)Devroye, Lerasle, Lugosi, and Oliveira]{MR3576558}
Devroye, L., Lerasle, M., Lugosi, G., and Oliveira, R.~I.
\newblock Sub-{G}aussian mean estimators.
\newblock \emph{The Annals of Statistics}, 44\penalty0 (6):\penalty0
  2695--2725, 2016.

\bibitem[Dheeru \& Karra~Taniskidou(2017)Dheeru and Karra~Taniskidou]{Dua:2017}
Dheeru, D. and Karra~Taniskidou, E.
\newblock {UCI} machine learning repository, 2017.
\newblock (\url{http://archive.ics.uci.edu/ml}).

\bibitem[Dziugaite et~al.(2015)Dziugaite, Roy, and
  Ghahramani]{dziugaite15training}
Dziugaite, G.~K., Roy, D.~M., and Ghahramani, Z.
\newblock Training generative neural networks via maximum mean discrepancy
  optimization.
\newblock In \emph{UAI}, pp.\  258--267, 2015.

\bibitem[Fukumizu et~al.(2008)Fukumizu, Gretton, Sun, and
  Sch{\"o}lkopf]{fukumizu08kernel}
Fukumizu, K., Gretton, A., Sun, X., and Sch{\"o}lkopf, B.
\newblock Kernel measures of conditional dependence.
\newblock In \emph{NIPS}, pp.\  498--496, 2008.

\bibitem[Fukumizu et~al.(2013)Fukumizu, Song, and Gretton]{fukumizu13kernel}
Fukumizu, K., Song, L., and Gretton, A.
\newblock Kernel {B}ayes' rule: Bayesian inference with positive definite
  kernels.
\newblock \emph{Journal of Machine Learning Research}, 14:\penalty0 3753--3783,
  2013.

\bibitem[G{\"a}rtner et~al.(2002)G{\"a}rtner, Flach, Kowalczyk, and
  Smola]{gartner02multi}
G{\"a}rtner, T., Flach, P.~A., Kowalczyk, A., and Smola, A.
\newblock Multi-instance kernels.
\newblock In \emph{ICML}, pp.\  179--186, 2002.

\bibitem[Gretton et~al.(2008)Gretton, Fukumizu, Teo, Song, Sch{\"o}lkopf, and
  Smola]{gretton08kernel}
Gretton, A., Fukumizu, K., Teo, C.~H., Song, L., Sch{\"o}lkopf, B., and Smola,
  A.~J.
\newblock A kernel statistical test of independence.
\newblock In \emph{NIPS}, pp.\  585--592, 2008.

\bibitem[Gretton et~al.(2012)Gretton, Borgwardt, Rasch, Sch{\"o}lkopf, and
  Smola]{gretton12kernel}
Gretton, A., Borgwardt, K.~M., Rasch, M.~J., Sch{\"o}lkopf, B., and Smola, A.
\newblock A kernel two-sample test.
\newblock \emph{Journal of Machine Learning Research}, 13:\penalty0 723--773,
  2012.

\bibitem[Guevara et~al.(2017)Guevara, Hirata, and Canu]{guevara17cross}
Guevara, J., Hirata, R., and Canu, S.
\newblock Cross product kernels for fuzzy set similarity.
\newblock In \emph{FUZZ-IEEE}, pp.\  1--6, 2017.

\bibitem[Gy{\"o}rfi et~al.(2002)Gy{\"o}rfi, Kohler, Krzyzak, and
  Walk]{gyorfi2002}
Gy{\"o}rfi, L., Kohler, M., Krzyzak, A., and Walk, H.
\newblock \emph{A Distribution-Free Theory of Nonparametric Regression}.
\newblock Springer, New-york, 2002.

\bibitem[Harchaoui \& Capp{\'e}(2007)Harchaoui and
  Capp{\'e}]{harchaoui07retrospective}
Harchaoui, Z. and Capp{\'e}, O.
\newblock Retrospective mutiple change-point estimation with kernels.
\newblock In \emph{IEEE/SP 14th Workshop on Statistical Signal Processing},
  pp.\  768--772, 2007.

\bibitem[Harchaoui et~al.(2007)Harchaoui, Bach, and
  Moulines]{harchaoui07testing}
Harchaoui, Z., Bach, F., and Moulines, E.
\newblock Testing for homogeneity with kernel {F}isher discriminant analysis.
\newblock In \emph{NIPS}, pp.\  609--616, 2007.

\bibitem[Haussler(1999)]{haussler99convolution}
Haussler, D.
\newblock Convolution kernels on discrete structures.
\newblock Technical report, Department of Computer Science, University of
  California at Santa Cruz, 1999.
\newblock
  (\url{http://cbse.soe.ucsc.edu/sites/default/files/convolutions.pdf}).

\bibitem[Hein \& Bousquet(2005)Hein and Bousquet]{hein05hilbertian}
Hein, M. and Bousquet, O.
\newblock Hilbertian metrics and positive definite kernels on probability
  measures.
\newblock In \emph{AISTATS}, pp.\  136--143, 2005.

\bibitem[Hopkins(2018)]{hopkins18mean}
Hopkins, S.~B.
\newblock Mean estimation with sub-{G}aussian rates in polynomial time.
\newblock Technical report, 2018.
\newblock (\url{https://arxiv.org/abs/1809.07425}).

\bibitem[Jebara et~al.(2004)Jebara, Kondor, and Howard]{jebara04probability}
Jebara, T., Kondor, R., and Howard, A.
\newblock Probability product kernels.
\newblock \emph{Journal of Machine Learning Research}, 5:\penalty0 819--844,
  2004.

\bibitem[Jerrum et~al.(1986)Jerrum, G.Valiant, and V.Vazirani]{MR855970}
Jerrum, M.~R., G.Valiant, L., and V.Vazirani, V.
\newblock Random generation of combinatorial structures from a uniform
  distribution.
\newblock \emph{Theoretical Computer Science}, 43\penalty0 (2-3):\penalty0
  169--188, 1986.

\bibitem[Jiao \& Vert(2016)Jiao and Vert]{jiao16kendall}
Jiao, Y. and Vert, J.-P.
\newblock The {K}endall and {M}allows kernels for permutations.
\newblock In \emph{ICML (PMLR)}, volume~37, pp.\  2982--2990, 2016.

\bibitem[Jitkrittum et~al.(2017)Jitkrittum, Xu, Szab{\'o}, Fukumizu, and
  Gretton]{jitkrittum17linear}
Jitkrittum, W., Xu, W., Szab{\'o}, Z., Fukumizu, K., and Gretton, A.
\newblock A linear-time kernel goodness-of-fit test.
\newblock In \emph{NIPS}, pp.\  261--270, 2017.

\bibitem[Kashima \& Koyanagi(2002)Kashima and Koyanagi]{kashima02kernels}
Kashima, H. and Koyanagi, T.
\newblock Kernels for semi-structured data.
\newblock In \emph{ICML}, pp.\  291--298, 2002.

\bibitem[Kim et~al.(2016)Kim, Khanna, and Koyejo]{kim16examples}
Kim, B., Khanna, R., and Koyejo, O.~O.
\newblock Examples are not enough, learn to criticize! criticism for
  interpretability.
\newblock In \emph{NIPS}, pp.\  2280--2288, 2016.

\bibitem[Kim \& Scott(2012)Kim and Scott]{kim12robust}
Kim, J. and Scott, C.~D.
\newblock Robust kernel density estimation.
\newblock \emph{Journal of Machine Learning Research}, 13:\penalty0 2529--2565,
  2012.

\bibitem[Klebanov(2005)]{klebanov05n-distance}
Klebanov, L.
\newblock \emph{N-Distances and Their Applications}.
\newblock Charles University, Prague, 2005.

\bibitem[Klus et~al.(2018)Klus, Schuster, and
  Muandet]{klus18eigendecompositions}
Klus, S., Schuster, I., and Muandet, K.
\newblock Eigendecompositions of transfer operators in reproducing kernel
  {H}ilbert spaces.
\newblock Technical report, 2018.
\newblock (\url{https://arxiv.org/abs/1712.01572}).

\bibitem[Klus et~al.(2019)Klus, Bittracher, Schuster, and
  Sch{\"u}tte]{klus19kernel}
Klus, S., Bittracher, A., Schuster, I., and Sch{\"u}tte, C.
\newblock A kernel-based approach to molecular conformation analysis.
\newblock \emph{The Journal of Chemical Physics}, 149:\penalty0 244109, 2019.

\bibitem[Koltchinskii(2011)]{koltchinskii11oracle}
Koltchinskii, V.
\newblock \emph{Oracle Inequalities in Empirical Risk Minimization and Sparse
  Recovery Problems}.
\newblock Springer, 2011.

\bibitem[Koltchinskii \& Mendelson(2015)Koltchinskii and Mendelson]{MR3431642}
Koltchinskii, V. and Mendelson, S.
\newblock Bounding the smallest singular value of a random matrix without
  concentration.
\newblock \emph{International Mathematics Research Notices}, \penalty0
  (23):\penalty0 12991--13008, 2015.

\bibitem[Kondor \& Pan(2016)Kondor and Pan]{kondor16multiscale}
Kondor, R. and Pan, H.
\newblock The multiscale {L}aplacian graph kernel.
\newblock In \emph{NIPS}, pp.\  2982--2990, 2016.

\bibitem[Kusano et~al.(2016)Kusano, Fukumizu, and Hiraoka]{kusano16persistence}
Kusano, G., Fukumizu, K., and Hiraoka, Y.
\newblock Persistence weighted {G}aussian kernel for topological data analysis.
\newblock In \emph{ICML}, pp.\  2004--2013, 2016.

\bibitem[Law et~al.(2018)Law, Sutherland, Sejdinovic, and
  Flaxman]{law18bayesian}
Law, H. C.~L., Sutherland, D.~J., Sejdinovic, D., and Flaxman, S.
\newblock Bayesian approaches to distribution regression.
\newblock \emph{AISTATS (PMLR)}, 84:\penalty0 1167--1176, 2018.

\bibitem[Le~Cam(1973)]{MR0334381}
Le~Cam, L.
\newblock Convergence of estimates under dimensionality restrictions.
\newblock \emph{The Annals of Statistics}, 1:\penalty0 38--53, 1973.

\bibitem[Lecu{\'e} \& Lerasle(2018)Lecu{\'e} and Lerasle]{LL1}
Lecu{\'e}, G. and Lerasle, M.
\newblock Learning from {MOM}'s principles: {L}e {C}am's approach.
\newblock \emph{Stochastic Processes and their Applications}, 2018.
\newblock (\url{https://doi.org/10.1016/j.spa.2018.11.024}).

\bibitem[Lecu{\'e} \& Lerasle(2019)Lecu{\'e} and Lerasle]{LL2}
Lecu{\'e}, G. and Lerasle, M.
\newblock Robust machine learning via median of means: theory and practice.
\newblock \emph{Annals of Statistics}, 2019.
\newblock (To appear; preprint: \url{https://arxiv.org/abs/1711.10306}).

\bibitem[Ledoux \& Talagrand(1991)Ledoux and Talagrand]{LT:91}
Ledoux, M. and Talagrand, M.
\newblock \emph{Probability in {B}anach spaces}.
\newblock Springer-Verlag, 1991.

\bibitem[Li et~al.(2015)Li, Swersky, and Zemel]{li15generative}
Li, Y., Swersky, K., and Zemel, R.
\newblock Generative moment matching networks.
\newblock In \emph{ICML (PMLR)}, pp.\  1718--1727, 2015.

\bibitem[Lloyd et~al.(2014)Lloyd, Duvenaud, Grosse, Tenenbaum, and
  Ghahramani]{loyd14automatic}
Lloyd, J.~R., Duvenaud, D., Grosse, R., Tenenbaum, J.~B., and Ghahramani, Z.
\newblock Automatic construction and natural-language description of
  nonparametric regression models.
\newblock In \emph{AAAI Conference on Artificial Intelligence}, pp.\
  1242--1250, 2014.

\bibitem[Lodhi et~al.(2002)Lodhi, Saunders, Shawe-Taylor, Cristianini, and
  Watkins]{lodhi02text}
Lodhi, H., Saunders, C., Shawe-Taylor, J., Cristianini, N., and Watkins, C.
\newblock Text classification using string kernels.
\newblock \emph{Journal of Machine Learning Research}, 2:\penalty0 419--444,
  2002.

\bibitem[Lugosi \& Mendelson(2019{\natexlab{a}})Lugosi and
  Mendelson]{LugosiMendelson2016}
Lugosi, G. and Mendelson, S.
\newblock Risk minimization by median-of-means tournaments.
\newblock \emph{Journal of the European Mathematical Society},
  2019{\natexlab{a}}.
\newblock (To appear; preprint: \url{https://arxiv.org/abs/1608.00757}).

\bibitem[Lugosi \& Mendelson(2019{\natexlab{b}})Lugosi and
  Mendelson]{LugosiMendelson2017-2}
Lugosi, G. and Mendelson, S.
\newblock Sub-gaussian estimators of the mean of a random vector.
\newblock \emph{Annals of Statistics}, 47\penalty0 (2):\penalty0 783--794,
  2019{\natexlab{b}}.

\bibitem[Martins et~al.(2009)Martins, Smith, Xing, Aguiar, and
  Figueiredo]{martins09nonextensive}
Martins, A. F.~T., Smith, N.~A., Xing, E.~P., Aguiar, P. M.~Q., and Figueiredo,
  M. A.~T.
\newblock Nonextensive information theoretic kernels on measures.
\newblock \emph{The Journal of Machine Learning Research}, 10:\penalty0
  935--975, 2009.

\bibitem[Mendelson(2015)]{Shahar-COLT}
Mendelson, S.
\newblock Learning without concentration.
\newblock \emph{Journal of the ACM}, 62\penalty0 (3):\penalty0 21:1--21:25,
  2015.

\bibitem[Minsker(2015)]{MR3378468}
Minsker, S.
\newblock Geometric median and robust estimation in {B}anach spaces.
\newblock \emph{Bernoulli}, 21\penalty0 (4):\penalty0 2308--2335, 2015.

\bibitem[Minsker \& Strawn(2017)Minsker and Strawn]{Minsker-Strawn2017}
Minsker, S. and Strawn, N.
\newblock Distributed statistical estimation and rates of convergence in normal
  approximation.
\newblock Technical report, 2017.
\newblock (\url{https://arxiv.org/abs/1704.02658}).

\bibitem[Mooij et~al.(2016)Mooij, Peters, Janzing, Zscheischler, and
  Sch{\"o}lkopf]{mooij16distinguishing}
Mooij, J.~M., Peters, J., Janzing, D., Zscheischler, J., and Sch{\"o}lkopf, B.
\newblock Distinguishing cause from effect using observational data: Methods
  and benchmarks.
\newblock \emph{Journal of Machine Learning Research}, 17:\penalty0 1--102,
  2016.

\bibitem[Muandet et~al.(2011)Muandet, Fukumizu, Dinuzzo, and
  Sch{\"o}lkopf]{muandet12learning}
Muandet, K., Fukumizu, K., Dinuzzo, F., and Sch{\"o}lkopf, B.
\newblock Learning from distributions via support measure machines.
\newblock In \emph{NIPS}, pp.\  10--18, 2011.

\bibitem[Muandet et~al.(2016)Muandet, Sriperumbudur, Fukumizu, Gretton, and
  Sch{\"o}lkopf]{muandet16kernel}
Muandet, K., Sriperumbudur, B.~K., Fukumizu, K., Gretton, A., and
  Sch{\"o}lkopf, B.
\newblock Kernel mean shrinkage estimators.
\newblock \emph{Journal of Machine Learning Research}, 17:\penalty0 1--41,
  2016.

\bibitem[Muandet et~al.(2017)Muandet, Fukumizu, Sriperumbudur, and
  Sch{\"o}lkopf]{maundet17kernel}
Muandet, K., Fukumizu, K., Sriperumbudur, B., and Sch{\"o}lkopf, B.
\newblock Kernel mean embedding of distributions: A review and beyond.
\newblock \emph{Foundations and Trends in Machine Learning}, 10\penalty0
  (1-2):\penalty0 1--141, 2017.

\bibitem[M{\"u}ller(1997)]{muller97integral}
M{\"u}ller, A.
\newblock Integral probability metrics and their generating classes of
  functions.
\newblock \emph{Advances in Applied Probability}, 29:\penalty0 429--443, 1997.

\bibitem[Nemirovski \& Yudin(1983)Nemirovski and Yudin]{MR702836}
Nemirovski, A.~S. and Yudin, D.~B.
\newblock \emph{Problem complexity and method efficiency in optimization}.
\newblock John Wiley \& Sons Ltd., 1983.

\bibitem[Park et~al.(2016)Park, Jitkrittum, and Sejdinovic]{park16k2abc}
Park, M., Jitkrittum, W., and Sejdinovic, D.
\newblock {K2-ABC}: Approximate {B}ayesian computation with kernel embeddings.
\newblock In \emph{AISTATS (PMLR)}, volume~51, pp.\  51:398--407, 2016.

\bibitem[Pfister et~al.(2017)Pfister, B{\"u}hlmann, Sch{\"o}lkopf, and
  Peters]{pfister17kernel}
Pfister, N., B{\"u}hlmann, P., Sch{\"o}lkopf, B., and Peters, J.
\newblock Kernel-based tests for joint independence.
\newblock \emph{Journal of the Royal Statistical Society: Series B (Statistical
  Methodology)}, 2017.

\bibitem[Raj et~al.(2018)Raj, Law, Sejdinovic, and Park]{raj18differentially}
Raj, A., Law, H. C.~L., Sejdinovic, D., and Park, M.
\newblock A differentially private kernel two-sample test.
\newblock Technical report, 2018.
\newblock (\url{https://arxiv.org/abs/1808.00380}).

\bibitem[Sch{\"o}lkopf \& Smola(2002)Sch{\"o}lkopf and
  Smola]{scholkopf02learning}
Sch{\"o}lkopf, B. and Smola, A.~J.
\newblock \emph{Learning with Kernels: Support Vector Machines, Regularization,
  Optimization, and Beyond}.
\newblock MIT Press, 2002.

\bibitem[Sch{\"o}lkopf et~al.(2001)Sch{\"o}lkopf, Herbrich, and
  Smola]{scholkopf01generalized}
Sch{\"o}lkopf, B., Herbrich, R., and Smola, A.~J.
\newblock A generalized representer theorem.
\newblock In \emph{COLT}, pp.\  416--426, 2001.

\bibitem[Sch{\"o}lkopf et~al.(2015)Sch{\"o}lkopf, Muandet, Fukumizu, Harmeling,
  and Peters]{scholkopf15computing}
Sch{\"o}lkopf, B., Muandet, K., Fukumizu, K., Harmeling, S., and Peters, J.
\newblock Computing functions of random variables via reproducing kernel
  {H}ilbert space representations.
\newblock \emph{Statistics and Computing}, 25\penalty0 (4):\penalty0 755--766,
  2015.

\bibitem[Sejdinovic et~al.(2013)Sejdinovic, Sriperumbudur, Gretton, and
  Fukumizu]{sejdinovic13equivalence}
Sejdinovic, D., Sriperumbudur, B.~K., Gretton, A., and Fukumizu, K.
\newblock Equivalence of distance-based and {RKHS}-based statistics in
  hypothesis testing.
\newblock \emph{Annals of Statistics}, 41:\penalty0 2263--2291, 2013.

\bibitem[Sinova et~al.(2018)Sinova, González-Rodr{\'\i}guez, and
  Aelst]{sinova18m-estimators}
Sinova, B., González-Rodr{\'\i}guez, G., and Aelst, S.~V.
\newblock M-estimators of location for functional data.
\newblock \emph{Bernoulli}, 24:\penalty0 2328--2357, 2018.

\bibitem[Smola et~al.(2007)Smola, Gretton, Song, and
  Sch{\"o}olkopf]{smola07hilbert}
Smola, A., Gretton, A., Song, L., and Sch{\"o}olkopf, B.
\newblock A {H}ilbert space embedding for distributions.
\newblock In \emph{ALT}, pp.\  13--31, 2007.

\bibitem[Song et~al.(2011)Song, Gretton, Bickson, Low, and
  Guestrin]{song11kernel}
Song, L., Gretton, A., Bickson, D., Low, Y., and Guestrin, C.
\newblock Kernel belief propagation.
\newblock In \emph{AISTATS}, pp.\  707--715, 2011.

\bibitem[Sriperumbudur et~al.(2010)Sriperumbudur, Gretton, Fukumizu,
  Sch{\"o}lkopf, and Lanckriet]{sriperumbudur10hilbert}
Sriperumbudur, B.~K., Gretton, A., Fukumizu, K., Sch{\"o}lkopf, B., and
  Lanckriet, G.~R.
\newblock Hilbert space embeddings and metrics on probability measures.
\newblock \emph{Journal of Machine Learning Research}, 11:\penalty0 1517--1561,
  2010.

\bibitem[Steinwart \& Christmann(2008)Steinwart and
  Christmann]{steinwart08support}
Steinwart, I. and Christmann, A.
\newblock \emph{Support Vector Machines}.
\newblock Springer, 2008.

\bibitem[Szab{\'o}(2014)]{szabo14information}
Szab{\'o}, Z.
\newblock Information theoretical estimators toolbox.
\newblock \emph{Journal of Machine Learning Research}, 15:\penalty0 283--287,
  2014.

\bibitem[Szab{\'o} et~al.(2016)Szab{\'o}, Sriperumbudur, P{\'o}czos, and
  Gretton]{szabo16learning}
Szab{\'o}, Z., Sriperumbudur, B., P{\'o}czos, B., and Gretton, A.
\newblock Learning theory for distribution regression.
\newblock \emph{Journal of Machine Learning Research}, 17\penalty0
  (152):\penalty0 1--40, 2016.

\bibitem[Sz{\'e}kely \& Rizzo(2004)Sz{\'e}kely and Rizzo]{szekely04testing}
Sz{\'e}kely, G.~J. and Rizzo, M.~L.
\newblock Testing for equal distributions in high dimension.
\newblock \emph{InterStat}, 5, 2004.

\bibitem[Sz{\'e}kely \& Rizzo(2005)Sz{\'e}kely and Rizzo]{szekely05new}
Sz{\'e}kely, G.~J. and Rizzo, M.~L.
\newblock A new test for multivariate normality.
\newblock \emph{Journal of Multivariate Analysis}, 93:\penalty0 58--80, 2005.

\bibitem[Tolstikhin et~al.(2016)Tolstikhin, Sriperumbudur, and
  Sch{\"o}lkopf]{tolstikhin16minimax}
Tolstikhin, I., Sriperumbudur, B.~K., and Sch{\"o}lkopf, B.
\newblock Minimax estimation of maximal mean discrepancy with radial kernels.
\newblock In \emph{NIPS}, pp.\  1930--1938, 2016.

\bibitem[Tolstikhin et~al.(2017)Tolstikhin, Sriperumbudur, and
  Muandet]{tolstikhin17minimax}
Tolstikhin, I., Sriperumbudur, B.~K., and Muandet, K.
\newblock Minimax estimation of kernel mean embeddings.
\newblock \emph{Journal of Machine Learning Research}, 18:\penalty0 1--47,
  2017.

\bibitem[Vandermeulen \& Scott(2013)Vandermeulen and
  Scott]{vandermeulen13consistency}
Vandermeulen, R. and Scott, C.
\newblock Consistency of robust kernel density estimators.
\newblock In \emph{COLT (PMLR)}, volume~30, pp.\  568--591, 2013.

\bibitem[Vapnik(2000)]{MR1719582}
Vapnik, V.~N.
\newblock \emph{The nature of statistical learning theory}.
\newblock Statistics for Engineering and Information Science. Springer-Verlag,
  New York, second edition, 2000.

\bibitem[Vishwanathan et~al.(2010)Vishwanathan, Schraudolph, Kondor, and
  Borgwardt]{vishwanathan10graph}
Vishwanathan, S.~N., Schraudolph, N.~N., Kondor, R., and Borgwardt, K.~M.
\newblock Graph kernels.
\newblock \emph{Journal of Machine Learning Research}, 11:\penalty0 1201--1242,
  2010.

\bibitem[Yamada et~al.(2018)Yamada, Umezu, Fukumizu, and
  Takeuchi]{yamada18post}
Yamada, M., Umezu, Y., Fukumizu, K., and Takeuchi, I.
\newblock Post selection inference with kernels.
\newblock In \emph{AISTATS (PMLR)}, volume~84, pp.\  152--160, 2018.

\bibitem[Zaheer et~al.(2017)Zaheer, Kottur, Ravanbakhsh, P{\'o}czos,
  Salakhutdinov, and Smola]{zaheer17deep}
Zaheer, M., Kottur, S., Ravanbakhsh, S., P{\'o}czos, B., Salakhutdinov, R.~R.,
  and Smola, A.~J.
\newblock Deep sets.
\newblock In \emph{NIPS}, pp.\  3394--3404, 2017.

\bibitem[Zhang et~al.(2013)Zhang, Sch{\"o}lkopf, Muandet, and
  Wang]{zhang13domain}
Zhang, K., Sch{\"o}lkopf, B., Muandet, K., and Wang, Z.
\newblock Domain adaptation under target and conditional shift.
\newblock \emph{Journal of Machine Learning Research}, 28\penalty0
  (3):\penalty0 819--827, 2013.

\bibitem[Zinger et~al.(1992)Zinger, Kakosyan, and
  Klebanov]{zinger92characterization}
Zinger, A.~A., Kakosyan, A.~V., and Klebanov, L.~B.
\newblock A characterization of distributions by mean values of statistics and
  certain probabilistic metrics.
\newblock \emph{Journal of Soviet Mathematics}, 1992.

\bibitem[Zolotarev(1983)]{zolotarev83probability}
Zolotarev, V.~M.
\newblock Probability metrics.
\newblock \emph{Theory of Probability and its Applications}, 28:\penalty0
  278--302, 1983.

\end{thebibliography}
\bibliographystyle{icml2019}

\clearpage
\appendix

\begin{center}
 {\Large\tb{Supplement}}
 \end{center}\vspace{0.2cm}

 The supplement contains the detailed proofs of our results (Section~\ref{sec:proofs}), a few technical lemmas used during these arguments (Section~\ref{sec:app:proofs}), the McDiarmid inequality for 
 self-containedness (Section~\ref{sec:app:ext-lemmas}), and the
 pseudocode of the two-sample test performed in Experiment-2 (Section~\ref{sec:experiment-2}).

 \section{Proofs of Theorem~\ref{thm:RB} and Theorem~\ref{thm:RB2}}\label{sec:proofs}
 This section contains the detailed proofs of Theorem~\ref{thm:RB} (Section~\ref{sec:proof:thm1}) and Theorem~\ref{thm:RB2} (Section~\ref{sec:proof:thm2}).
\subsection{Proof of Theorem~\ref{thm:RB}}\label{sec:proof:thm1}
 The structure of the proof is as follows:
\begin{enumerate}[labelindent=0cm,leftmargin=*,topsep=0cm,partopsep=0cm,parsep=0cm,itemsep=0cm]
  \item We show that $\left\|\hat{\mu}_{\P,Q} - \mu_\P\right\|_{\k} \le (1+\sqrt{2})r_{Q,N}$, where
    $r_{Q,N}  = \sup_{f\in B_\k}\text{MON}_Q \Big[ \underbrace{\psh{f}{\k(\cdot,x)-\mu_\P}_{\k}}_{f(x) - \P f}\Big]$, i.e.\ the analysis can be reduced to $B_\k$.
  \item Then $r_{Q,N}$ is bounded using empirical processes.
\end{enumerate}

\tb{Step-1}:
Since $\Hk$ is an inner product space, for any $f\in \Hk$
\begin{align}
    &\norm{f-\k(\cdot,x)}_{\k}^2-\norm{\mu_\P-\k(\cdot,x)}_{\k}^2 \nonumber\\
    & =  \norm{f-\mu_\P}_{\k}^2 - 2\psh{f- \mu_\P}{\k(\cdot,x)-\mu_\P}_{\k}. \label{eq:MQDec}
\end{align}
Hence, by denoting $e =\hat{\mu}_{\P,Q}-\mu_\P$, $\tilde{g} = g - \mu_\P$ we get
\begin{align}
  &\norm{e}_{\k}^2-2r_{Q,N}\norm{e}_{\k} \nonumber\\
  &\stackrel{(a)}{\le}   \norm{e}_{\k}^2 - 2\MOM{Q}{\pscal{\frac{e}{\norm{e}_{\k}}}{\k(\cdot,x) - \mu_\P}}_{\k}\norm{e}_{\k}\nonumber\\
  &\stackrel{(b)}{\le} \MOM{Q}{\norm{e}_{\k}^2 - 2 \pscal{\frac{e}{\norm{e}_{\k}}}{\k(\cdot,x) - \mu_\P}_{\k} \norm{e}_{\k}}\nonumber\\
 &\stackrel{(c)}{\le} \MOM{Q}{\norm{\hat{\mu}_{\P,Q} - \k(\cdot,x)}_{\k}^2-\norm{\mu_\P - \k(\cdot,x)}_{\k}^2}\nonumber\\
 &\stackrel{(d)}{\le} \sup_{g\in \Hk}\MOM{Q}{\norm{\hat{\mu}_{\P,Q} - \k(\cdot,x)}_{\k}^2-\norm{g - \k(\cdot,x)}_{\k}^2}\nonumber\\
 &\stackrel{(e)}{\le} \sup_{g\in \Hk}\MOM{Q}{\norm{\mu_\P - \k(\cdot,x)}_{\k}^2-\norm{g - \k(\cdot,x)}_{\k}^2}\nonumber\\
 & \stackrel{(f)} = \sup_{g\in \Hk}\Big\{2\text{MON}_{Q}\Big[{\underbrace{\psh{\tilde{g}}{\k(\cdot,x) - \mu_\P}_{\k}}_{\norm{\tilde{g}}_{\k}\psh{\frac{\tilde{g}}{\norm{\tilde{g}}_{\k}}}{\k(\cdot,x) - \mu_\P}_{\k}}}\Big] - \norm{\tilde{g}}_{\k}^2\Big\}\nonumber\\
 &\stackrel{(g)}{=}\sup_{g\in \Hk}\set{2\norm{\tilde{g}}_{\k}r_{Q,N} - \norm{\tilde{g}}_{\k}^2} \stackrel{(h)}{\leq} r_{Q,N}^2, \label{eq:r}
\end{align}
where we used in (a)  the definition of $r_{Q,N}$, (b) the linearity\footnote{$\MOM{Q}{c_1+c_2f} = c_1 + c_2 \MOM{Q}{f}$ for any $c_1, c_2\in \R$.} of $\MOM{Q}{\cdot}$, (c) Eq.~\eqref{eq:MQDec}, (d) $\sup_g$, (e) the definition of $\hat{\mu}_{\P,Q}$, (f) Eq.~\eqref{eq:MQDec} and the
linearity of $\MOM{Q}{\cdot}$, (g) the definition of $r_{Q,N}$. In step (h), by denoting $a=\norm{\tilde{g}}_{\k}$, $r=r_{Q,N}$, the argument of the $\sup$ takes the form $2ar-a^2$; $2ar-a^2 \le r^2 \Leftrightarrow 0 \le r^2-2ar+a^2 = (r-a)^2$.

In Eq.~\eqref{eq:r}, we obtained an equation $a^2-2ra \le r^2$ where $a:=\norm{e}_{\k}\ge 0$. Hence
$r^2 +2ra -a^2\ge 0$, $r_{1,2} =  \left[-2a\pm \sqrt{4a^2+4a^2}\right]/2 = \left(-1 \pm \sqrt{2}\right)a$, thus by the non-negativity of $a$, $r\ge (-1+\sqrt{2})a$, i.e.,
$a \le \frac{r}{\sqrt{2}-1} = (\sqrt{2}+1)r$. In other words, we arrived at
\begin{align}\label{eq:TrivBound}
 \norm{\hat{\mu}_{\P,Q} - \mu_\P}_{\k} &\leq \left(1+\sqrt{2}\right)r_{Q,N}.
\end{align}
It remains to upper bound $r_{Q,N}$.

\tb{Step-2}: Our goal is to provide a probabilistic bound on
\begin{align*}
    & r_{Q,N}=\sup_{f\in B_\k}\MOM{Q}{x\mapsto\psh{f}{\k(\cdot,x)-\mu_\P}_{\k}}\\
    & = \sup_{f\in B_\k} \med{q \in [Q]}\{\underbrace{\pscal{f}{\mu_{S_q}-\mu_\P}_{\k}}_{=:r(f,q)}\}.
\end{align*}
The $N_c$ corrupted samples can affect (at most) $N_c$ of the $(S_q)_{q\in [Q]}$ blocks. Let $U := [Q] \backslash C$ stand for the indices of the uncorrupted sets, where
$C:=\{q\in [Q]: \exists n_j\text{ s.t. } n_j \in S_q,\, j\in[N_c]\}$ contains the indices of the corrupted sets.
If
\begin{align}
  \forall f \in B_\k: \underbrace{\left|\left\{q\in U: r(f,q)\ge \epsilon \right\}\right|}_{\sum_{q\in U} \I_{r(f,q) \ge \epsilon}} + N_c \le \frac{Q}{2}, \label{eq:good-event}
\end{align}
then for $\forall f \in B_\k$, $\medi{q \in [Q]}\{r(f,q)\} \le \epsilon$, i.e.\
$\sup_{f\in B_\k} \medi{q \in [Q]}\{r(f,q) \} \le \epsilon$. Thus, our task boils down to controlling the event in \eqref{eq:good-event}
by appropriately choosing $\epsilon$.

\begin{itemize}[labelindent=0cm,leftmargin=*,topsep=0cm,partopsep=0cm,parsep=0cm,itemsep=0cm]
 \item \tb{Controlling} $r(f,q)$: For any $f\in B_K$ the random variables $\psh{f}{k(\cdot,x_i)-\mu_\P}_{\Hk}=f(x_i)-\P f$ are independent, have zero mean, and
    \begin{align}
  &\E_{x_i\sim \P} \pscal{f}{k(\cdot,x_i)-\mu_\P}_{\k}^2  = \pscal{f}{\Sigma_{\P}f}_{\k}\nonumber \\
  &\le
   \left\|f\right\|_{\k} \left\|\Sigma_\P f\right\|_{\k}
      \le \left\|f\right\|_{\k}^2 \left\|\Sigma_\P\right\|  = \left\|\Sigma_\P\right\| \label{eq:1-term}
      \end{align}
      using the reproducing property of the kernel and the covariance operator, the Cauchy-Schwarz (CBS) inequality and $\left\|f\right\|_{\Hk}=1$.

      For a zero-mean random variable $z$ by the Chebyshev's inequality $\P \left( z > a \right) \le \P \left( |z| > a \right) \le  \E\left(z^2\right) / a^2$, which implies
      $\P\left(z> \sqrt{\E \left(z^2\right)/ \alpha}\right)\le \alpha$ by a $\alpha = \E\left(z^2\right)/a^2$ substitution.
      With $z := r(f,q)$ ($q\in U$), using $\E\left[z^2\right] = \E \pscal{f}{\mu_{S_q} - \mu_\P}_{\k}^2 =  \frac{Q}{N}\E_{x_i\sim \P} \pscal{f}{k(\cdot,x_i)-\mu_\P}_{\k}^2$ and
      Eq.~\eqref{eq:1-term} one gets that for all $f\in B_\k$, $\alpha \in (0,1)$ and $q \in U$:
      $\P\left( r(f,q) > \sqrt{\frac{\left\|\Sigma_{\P} \right\| Q}{\alpha N}} \right) \le \alpha$. 
      This means $\P\left(r(f,q) > \frac{\epsilon}{2}\right) \le \alpha$ with $\epsilon \ge 2 \sqrt{\frac{\left\|\Sigma_{\P} \right\| Q}{\alpha N}}$.
    \item \tb{Reduction to $\phi$}: As a result
  \begin{equation*}
  \sum_{q\in U} \P\left(r(f,q)  \ge \frac{\epsilon}{2}\right) \le |U|\alpha
  \end{equation*}
  happens if and only if 
\begin{align*}
  &\sum_{q \in U} \I_{r(f,q)  \ge  \epsilon} \\
  & \le |U|\alpha+
   \sum_{q \in U} \Big[ \I_{r(f,q) \ge  \epsilon}
    -\underbrace{\P\left(r(f,q) \ge \frac{\epsilon}{2}\right)}_{\E \left[\I_{r(f,q) \ge \frac{\epsilon}{2}}\right]}\Big] =: A.
  \end{align*}
  Let us introduce $\phi:t\in\R\to(t-1) \I_{1\le t \le 2} + \I_{t\ge 2}$.
  $\phi$ is $1$-Lipschitz and satisfies $ \I_{2\le t} \le \phi(t) \le \I_{1\le t}$ for any $t\in \R$. Hence, we can upper bound
  $A$ as
  \begin{align*}
      A &\le |U| \alpha + \sum_{q\in U} \left[\phi\left(\frac{2r(f,q)}{\epsilon}\right) -  \E \phi\left(\frac{2r(f,q)}{\epsilon}\right)\right]
  \end{align*}
  by noticing that $\epsilon \le r(f,q) \Leftrightarrow 2  \le 2r(f,q)/\epsilon$ and $\epsilon/2 \le r(f,q) \Leftrightarrow 1 \le 2r(f,q) / \epsilon$, and
  by using the $\I_{2\le t} \le \phi(t)$ and the $\phi(t) \le \I_{1\le t}$ bound, respectively. Taking supremum over $B_\k$ we arrive at
  \begin{align*}
      &\sup_{f\in B_\k}\sum_{q \in U} \I_{r(f,q)\ge \epsilon} \\ 
      &\le
       |U|\alpha + \underbrace{\sup_{f\in B_\k} \sum_{q\in U} \left[\phi\left(\frac{2r(f,q)}{\epsilon}\right) -  \E \phi\left(\frac{2r(f,q)}{\epsilon}\right)\right]}_{=:Z}.
  \end{align*}
    \item \tb{Concentration of $Z$ around its mean:} Notice that $Z$ is a function of $x_V$, the samples in the uncorrupted blocks; $V = \cup_{q\in U} S_q$.
    By the bounded difference property of $Z$ (Lemma~\ref{lemma:Z:bounded-diff}) for any $\beta > 0$,
    the McDiarmid inequality (Lemma~\ref{lemma:McDiarmid}; we choose $\tau := Q \beta^2 / 8$ to get linear scaling in $Q$ on the r.h.s.) implies that
    \begin{align*}
        \P\left( Z < \E_{x_V}[Z] + Q  \beta\right) \ge 1 - e^{-\frac{Q \beta^2}{8}}.
    \end{align*}
    \item \tb{Bounding $\E_{x_V}[Z]$}: Let $M=N/Q$ denote the number of elements in $S_q$-s. The $\G = \{g_f: f\in B_{\k}\}$ class with $g_f: \X^M \rightarrow \R$ and $\P_M := \frac{1}{M} \sum_{m=1}^M\delta_{u_m}$ defined as
    \begin{align*}
    g_f(u_{1:M}) &= \phi\left(\frac{\pscal{f}{\mu_{\P_M}-\mu_{\P}}_{\k}}{\epsilon}\right)
    \end{align*}
    is uniformly bounded separable Carath{\'e}odory (Lemma~\ref{eq:UBSC}), hence the symmetrization technique \citep[Prop.~7.10]{steinwart08support}, \citep{LT:91} gives
    \begin{align*}
        \E_{x_V}[Z] \le 2 \E_{x_V}\E_{\b{e}}  \sup_{f\in B_{\k}} \left| \sum_{q \in U} e_q \phi\left(\frac{2r(f,q)}{\epsilon} \right)\right|,
    \end{align*}
    where $\b{e} = (e_q)_{q\in U}\in\R^{|U|}$ with i.i.d.\ Rademacher  entries [$\P(e_q=\pm 1)=\frac{1}{2}$ ($\forall q$)].
    \item \tb{Discarding $\phi$}: Since $\phi(0) = 0$ and $\phi$ is 1-Lipschitz, by Talagrand's contraction principle of Rademacher processes \cite{LT:91}, \citep[Theorem~2.3]{koltchinskii11oracle} one gets
  \begin{align*}
      & \E_{x_V} \E_{\b{e}}  \sup_{f\in B_{\k}} 
      \left| \sum_{q \in U} e_q \phi\left(\frac{2r(f,q)}{\epsilon} \right)\right| \\ 
      &\le
      2 \E_{x_V}\E_{\b{e}}  \sup_{f\in B_{\k}} \left| \sum_{q \in U} e_q \frac{2r(f,q)}{\epsilon} \right|.
  \end{align*}
    \item \tb{Switching from $|U|$ to $N$ terms}: Applying an other symmetrization [(a)], the CBS inequality, $f\in B_\k$, and the Jensen inequality
  \begin{align*}
      &\E_{x_V}\E_{\b{e}}   \sup_{f\in B_{\k}} \left| \sum_{q=1}^Q e_q \frac{r(f,q)}{\epsilon} \right| \\
      & \stackrel{(a)}{\le} \frac{2Q}{\epsilon N}  \E_{x_{V}}\E_{\b{e}'} \Bigg[ \sup_{f\in B_\k} \Big| \underbrace{\sum_{n\in V} e_n' \left<f,\k(\cdot,x_n) - \mu_\P\right>_{\k}}_{= \left<f,\sum_{n \in V} e_n' \left[\k(\cdot,x_n) - \mu_\P\right]\right>_{\k}}\Big| \Bigg]\\
      &\le  \frac{2Q}{\epsilon N}  \E_{x_{V}}\E_{\b{e}'} \left[ \sup_{f\in B_\k} \underbrace{\left\|f\right\|_{\k}}_{=1} \left\| \sum_{n\in V} e_n' \left[\k(\cdot,x_n) - \mu_\P\right]\right\|_{\k} \right] \\
      & =  \frac{2Q}{\epsilon N}  \E_{x_{V}}\E_{\b{e}'}   \left\| \sum_{n\in V} e_n' \left[\k(\cdot,x_n) - \mu_\P\right]\right\|_{\k} \\
      &\le \frac{2Q}{\epsilon N}  \sqrt{ \E_{x_{V}}\E_{\b{e}'}   \left\| \sum_{n \in V} e_n' \left[\k(\cdot,x_n) - \mu_\P\right]\right\|_{\k}^2} \\
      & \stackrel{(b)}{=} \frac{2Q \sqrt{|V|\Tr(\Sigma_\P)}}{\epsilon N}.
  \end{align*}
  In (a), we proceed as follows:
  \begin{align*}
      &\E_{x_V}\E_{\b{e}}   \sup_{f\in B_{\k}} \left| \sum_{q \in U} e_q \frac{r(f,q)}{\epsilon} \right| \\
      &= \E_{x_V}\E_{\b{e}}   \sup_{f\in B_{\k}} \left| \sum_{q \in U} e_q \frac{\left<f,\mu_{S_q}-\mu_\P\right>_{\k}}{\epsilon} \right|\\
      & \stackrel{(c)}{\le} \frac{2Q}{N \epsilon} \E_{x_V} \E_{\b{e}} \E_{\b{e}'}  \sup_{f\in B_{\k}} \left| \sum_{n \in V} e_n' e_n'' \left<f,\k(\cdot,x_n)-\mu_\P\right>_{\k} \right|\\
       &= \frac{2Q}{N \epsilon} \E_{x_V} \E_{\b{e}'}  \sup_{f\in B_{\k}} \left| \sum_{n\in V} e_n' \left<f,\k(\cdot,x_n)-\mu_\P\right>_{\k} \right|,
  \end{align*}
  where in (c) we applied symmetrization, $\b e' = (e_n')_{n \in V}\in \R^{|V|}$ with i.i.d.\ Rademacher entries, $e_n'' = e_q$ if $n \in S_q$ ($q\in U$), and we used that
  $\left(e_n' e_n'' \left<f,\k(\cdot,x_n)-\mu_\P\right>_{\k} \right)_{n\in V} \stackrel{\text{distr}}{=} \left( e_n' \left<f,\k(\cdot,x_n)-\mu_\P\right>_{\k} \right)_{n \in V}$.
  
  In step (b), we had
  \begin{align*}
       &\E_{x_{V}}\E_{\b{e}'}   \left\| \sum\nolimits_{n \in V} e_n' \left[\k(\cdot,x_n) - \mu_\P\right]\right\|_{\k}^2\\
       & = \E_{x_{V}}\E_{\b{e}'} \sum_{n\in V}  \left[e_{n}'\right]^2 \left<\k(\cdot,x_{n}) - \mu_\P, \k(\cdot,x_{n}) - \mu_\P\right>_{\k}\\
        &= |V| \E_{x \sim \P} \left<\k(\cdot,x) - \mu_\P, \k(\cdot,x) - \mu_\P\right>_{\k}\\
        & = |V| \E_{x \sim \P} \Tr\left( [\k(\cdot,x) - \mu_\P] \otimes [\k(\cdot,x) - \mu_\P] \right)\\
        &= |V| \Tr(\Sigma_\P)
  \end{align*}
  exploiting the independence of $e_n'$-s and  $[e_n']^2 = 1$.
\end{itemize}
      Until this point we showed that for all $\alpha\in(0,1)$, $\beta>0$, if $\epsilon \ge 2 \sqrt{\frac{\left\|\Sigma_{\P} \right\| Q}{\alpha N}}$ then
      \begin{align*}
  \sup_{f\in B_\k} \sum_{q=1}^Q \I_{r(f,q)\ge \epsilon} \le |U|\alpha + Q\beta + \frac{8Q \sqrt{|V|  \Tr(\Sigma_{\P})}}{\epsilon N}
      \end{align*}
      with probability at least $1-e^{-\frac{Q\beta^2}{8}}$. Thus, to ensure that $\sup_{f\in B_\k} \sum_{q=1}^Q \I_{r(f,q)\ge \epsilon} + N_c\le Q / 2$ it is sufficient to choose
      $(\alpha,\beta,\epsilon)$ such that
    $|U|\alpha +  Q \beta + \frac{8Q \sqrt{|V|\Tr(\Sigma_{\P})}}{\epsilon N} + N_c \le \frac{Q}{2}$,
      and in this case $\left\|\hat{\mu}_{\P,Q} - \mu_\P\right\|_{\k}\le (1+\sqrt{2})\epsilon$. Applying the $|U| \le  Q$ and  $|V| \le N$ bounds, we want to have
      \begin{align}
    Q\alpha +  Q \beta + \frac{8Q \sqrt{\Tr(\Sigma_{\P})}}{\epsilon \sqrt{N}} + N_c \le \frac{Q}{2}. \label{eq:Q-constraint}
      \end{align}
      Choosing $\alpha = \beta = \frac{\delta}{3}$ in Eq.~\eqref{eq:Q-constraint}, the sum of the first two terms is $Q \frac{2\delta}{3}$;
      $\epsilon \ge \max\left(2 \sqrt{\frac{3\left\|\Sigma_{\P} \right\| Q}{\delta N}},\frac{24}{\delta}\sqrt{\frac{\Tr{\left(\Sigma_\P\right)}}{N}}\right)$ gives $\le Q\frac{\delta}{3}$ for the third term.
      Since $N_c  \le Q(\frac{1}{2}-\delta)$, we got
      \begin{align*}
  \left\|\hat{\mu}_{\P,Q} - \mu_\P\right\|_{\k} \le c_1 \max\left(\sqrt{\frac{3\left\|\Sigma_{\P} \right\| Q}{\delta N}},\frac{12}{\delta}\sqrt{\frac{\Tr{(\Sigma_\P)}}{N}}\right)
      \end{align*}
      with probability at least $1 - \rme^{-\frac{Q \delta^2}{72}}$. With an $\eta = \rme^{-\frac{Q \delta^2}{72}}$, and hence $Q=\frac{72 \ln\left(\frac{1}{\eta}\right)}{\delta^2}$
      reparameterization Theorem~\ref{thm:RB} follows.

\subsection{Proof of Theorem~\ref{thm:RB2}} \label{sec:proof:thm2}
The reasoning is similar to Theorem~\ref{thm:RB}; we detail the differences below. The high-level structure of the proof is as follows:
\begin{itemize}[labelindent=0cm,leftmargin=*,topsep=0cm,partopsep=0cm,parsep=0cm,itemsep=0cm]
    \item First we prove that
     $\big|\MMDQhat(\P,\Q) - \MMD(\P,\Q) \big| \le r_{Q,N}$,
     where $r_{Q,N} = \sup\limits_{f\in B_\k} \hspace{-0.1cm} \Big| \med{q\in[Q]} \hspace{-0.1cm} \left\{\left<f,\left(\mu_{S_q,\P}-\mu_{S_q,\Q}\right) - (\mu_\P-\mu_\Q)\right>_{\k} \right\}\hspace{-0.1cm} \Big|$.
    \item Then $r_{Q,N}$ is bounded.
\end{itemize}

\vspace{0.1cm}
    \tb{Step-1}:
      \begin{itemize}[labelindent=0cm,leftmargin=*,topsep=0cm,partopsep=0cm,parsep=0cm,itemsep=0cm]
    \item $\MMDQhat(\P,\Q)  - \MMD(\P,\Q) \le r_{Q,N}$:
    By the subadditivity of supremum [$\sup_f (a_f  + b_f) \le \sup_f a_f + \sup_f b_f$] one gets
    \begin{align*}
        &\MMDQhat(\P,\Q) \\
        &= \sup_{f\in B_\k} \med{q\in[Q]}\big\{\big<f,\left(\mu_{S_q,\P}-\mu_{S_q,\Q}\right) - (\mu_\P-\mu_\Q)\\ & \hspace{1cm} + (\mu_\P-\mu_\Q)\big>_{\k} \big\}\\
        &\le \sup_{f\in B_\k} \med{q\in[Q]}\left\{\left<f,\left(\mu_{S_q,\P}-\mu_{S_q,\Q}\right) - (\mu_\P-\mu_\Q)\right>_{\k}\right\}\\ & \hspace{1cm} + \sup_{f\in B_\k} \left<f,\mu_\P-\mu_\Q\right>_{\k}\\
        &\le \underbrace{\sup_{f\in B_\k} \left| \med{q\in[Q]}\left\{\left<f,\left(\mu_{S_q,\P}-\mu_{S_q,\Q}\right) - (\mu_\P-\mu_\Q)\right>_{\k}\right\} \right|}_{=r_{Q,N}}\\
        & \hspace{1cm}+ \MMD(\P,\Q).
    \end{align*}
    \item $\MMDQ(\P,\Q) - \MMDQhat(\P,\Q) \le r_{Q,N}$:
    Let $a_f := \left<f, \mu_\P - \mu_\Q\right>_{\k}$ and  $b_f := \med{q\in [Q]}\left\{\left<f,(\mu_\P-\mu_\Q) - (\mu_{S_q,\P} - \mu_{S_q,\Q})\right>_{\k} \right\}$.
    Then
    \begin{align*}
      &a_f - b_f\\ 
      &= \left<f, \mu_\P - \mu_\Q\right>_{\k} \\
      & \hspace{0.5cm} +\medi{q\in [Q]}\left\{\left<f,(\mu_{S_q,\P} - \mu_{S_q,\Q}) - (\mu_\P-\mu_\Q)\right>_{\k}\right\}\\
          & = \medi{q\in [Q]}\left\{\left<f,\mu_{S_q,\P} - \mu_{S_q,\Q}\right>_{\k}\right\}
    \end{align*}
        by $\medi{q\in [Q]}\{-z_q\} = - \medi{q\in[Q]}\{z_q\}$. Applying the $\sup_f (a_f-b_f) \ge \sup_f a_f - \sup_f b_f$ inequality (it follows from the subadditivity of $\sup$):
    \begin{align*}
        &\MMDQhat(\P,\Q)\\  
        &\ge \MMD(\P,\Q) \\
        & \hspace{0.5cm} - \sup_{f\in B_{\k}}\underbrace{\med{q\in [Q]}\big\{\big<f,(\mu_\P-\mu_\Q) - (\mu_{S_q,\P} - \mu_{S_q,\Q})\big>_{\k}\big\}}_{-\med{q\in [Q]}\left\{\left<f,(\mu_{S_q,\P} - \mu_{S_q,\Q}) - (\mu_\P-\mu_\Q) \right>_{\k} \right\}}\\
        &\ge \MMD(\P,\Q) \\
        & \hspace{0.5cm} -\underbrace{\sup_{f\in B_{\k}} \left| \med{q\in [Q]}\left\{\left<f,(\mu_{S_q,\P} - \mu_{S_q,\Q}) - (\mu_\P-\mu_\Q) \right>_{\k}\right\}\right|}_{r_{Q,N}}.
    \end{align*}
      \end{itemize}
    \tb{Step-2}: Our goal is to control
  \begin{align*}
      r_{Q,N}  &= \sup\nolimits_{f\in B_\k} \Big |\medi{q \in [Q]} \Big\{ r(f,q)\Big\}\Big|, \text{ where}\\
      r(f,q) &:=\left<f,(\mu_{S_q,\P}-\mu_{S_q,\Q}) - (\mu_\P-\mu_\Q)\right>_{\k}.
  \end{align*}
  The relevant quantities which change compared to the proof of Theorem~\ref{thm:RB} are as follows.
  \begin{itemize}[labelindent=0cm,leftmargin=*,topsep=0cm,partopsep=0cm,parsep=0cm,itemsep=0cm]
        \item \tb{Median rephrasing}:
        \begin{align*}
      &\sup_{f\in B_\k} \Big| \med{q \in [Q]}\{r(f,q) \}\Big| \le \epsilon \\
      & \Leftrightarrow
      \forall f\in B_{\k}: -\epsilon  \le \medi{q \in [Q]}\{r(f,q) \} \le \epsilon\\
      &\Leftarrow \forall f\in B_{\k}: \left| \left\{q: r(f,q) \le -\epsilon\right\} \right| \le Q/2 \\
      &\hspace{0.5cm} \text{ and  }
        \left| \left\{q: r(f,q) \ge \epsilon\right\} \right| \le Q/2\\
        &\Leftarrow \forall f\in B_{\k}: \left| \left\{q: |r(f,q)| \ge \epsilon\right\} \right| \le Q/2.
        \end{align*}
      Thus,   
        $\forall f \in B_\k: \left|\left\{q\in U: |r(f,q)|\ge \epsilon \right\}\right| + N_c \le \frac{Q}{2}$,
      implies $\sup_{f\in B_\k} \Big| \medi{q \in [Q]}\{r(f,q) \}\Big| \le \epsilon$.
        \item \tb{Controlling $|r(f,q)|$}:  For any $f\in B_\k$ the random variables $[f(x_i) - f(y_i)]-[\P f - \Q f]$ are independent, zero-mean and
      \begin{align*}
          &\E_{(x,y) \sim \P \otimes \Q} ([f(x) - \P f] - [f(y) -\Q f])^2\\
           & = \E_{x \sim \P} [f(x)-\P f]^2 + \E_{y\sim \Q} [f(y)-\Q f]^2 \\
          &\le \left\|\Sigma_\P\right\| + \left\|\Sigma_\Q\right\|,
      \end{align*}
      where $\P \otimes \Q$ is the product measure.
    The Chebyshev argument with $z=|r(f,q)|$ implies that $\forall \alpha \in (0,1)$
      \begin{align*}
          (\P\otimes  \Q)\left(|r(f,q)|  > \sqrt{\frac{\left(\left\|\Sigma_\P\right\| + \left\|\Sigma_\Q\right\|\right) Q}{\alpha N }}\right) \le \alpha.
     \end{align*}
     This means $(\P\otimes \Q)\left(|r(f,q)| > \epsilon / 2 \right) \le \alpha$ with $\epsilon \ge 2 \sqrt{\frac{\left(\left\|\Sigma_\P\right\| + \left\|\Sigma_\Q\right\|\right) Q}{\alpha N }}$.
     \item \tb{Switching from $|U|$ to $N$ terms}: With $(xy)_V= \left\{ (x_i,y_i): i\in V\right\}$, in '(b)' with $\tilde{x}_n :=\k(\cdot,x_n) - \mu_\P$, $\tilde{y}_n:=\k(\cdot,y_n)-\mu_\Q$ we arrive at
    \begin{align*}
       &\E_{(xy)_{V}}\E_{\b{e}'}   \left\| \sum_{n \in V} e_n' \left(\tilde{x}_n - \tilde{y}_n \right)\right\|_{\k}^2 \\
       &= \E_{(xy)_{V}}\E_{\b{e}'} \sum_{n\in V}  \left[e_{n}'\right]^2 \left< \tilde{x}_n - \tilde{y}_n, \tilde{x}_n - \tilde{y}_n\right>_{\k}\\
        &= |V| \E_{(xy) \sim \P} \left\|\left[\k(\cdot,x) - \mu_\P\right]- \left[\k(\cdot,y)-\mu_\Q\right]\right\|_{\k}\\
        &= |V| \left[ \Tr\left(\Sigma_\P\right) + \Tr\left(\Sigma_\Q\right) \right].
      \end{align*}
      \item These results imply
        \begin{align*}
        Q\alpha +  Q \beta + \frac{8Q \sqrt{\Tr\left(\Sigma_{\P}\right) + \Tr\left(\Sigma_{\Q}\right)}}{\epsilon \sqrt{N}} + N_c \le Q / 2.
        \end{align*}
       $\epsilon \ge \max\left(2 \sqrt{\frac{3\left(\left\|\Sigma_{\P}\right\| + \left\|\Sigma_{\Q} \right\|\right)  Q}{\delta N}},\frac{24}{\delta}\sqrt{\frac{\Tr{\left(\Sigma_\P\right)}+\Tr{\left(\Sigma_\Q\right)}}{N}}\right)$, $\alpha=\beta=\frac{\delta}{3}$ choice
       gives that
       \begin{align*}
    &\left|\MMDQhat(\P,\Q) - \MMD(\P,\Q) \right|\\
    &\le 2 \max\left(\sqrt{\frac{3\left( \left\|\Sigma_{\P} \right\|  + \left\|\Sigma_{\Q} \right\|\right) Q}{\delta N}},\frac{12}{\delta}\sqrt{\frac{\Tr{(\Sigma_\P)} + \Tr{(\Sigma_\Q)}}{N}}\right)
       \end{align*}
        with probability at least $1 - e^{-\frac{Q \delta^2}{72}}$. $\eta = e^{-\frac{Q \delta^2}{72}}$, i.e.\
        $Q=\frac{72 \ln\left(\frac{1}{\eta}\right)}{\delta^2}$     reparameterization finishes the proof of Theorem~\ref{thm:RB2}.
      \end{itemize}

\section{Technical Lemmas}\label{sec:app:proofs}
\begin{lemma}[Supremum] \label{lemma:sup}
    \begin{align*}
    \Big|\sup_f a_f - \sup_f b_f\Big| \le \sup_f |a_f-b_f|.
    \end{align*}
\end{lemma}

\begin{lemma}[Bounded difference property of $Z$] \label{lemma:Z:bounded-diff}
    Let $N\in \Z^+$, $(S_q)_{q \in [Q]}$ be a partition of $[N]$,
    $\k: \X \times \X \rightarrow \R$ be a kernel, $\mu$ be the mean embedding associated to $\k$, $x_{1:N}$ be i.i.d.\ random variables on $\X$,
    $Z(x_V) = \sup\limits_{f\in B_\k}\sum\limits_{q \in U} \left[ \phi\left(\frac{2\left<f,\mu_{S_q}-\mu_\P\right>_{\k}}{\epsilon}\right) - \E \phi\left(\frac{2\left<f,\mu_{S_q}-\mu_\P\right>_{\k}}{\epsilon}\right) \right]$, where $U\subseteq [Q]$, $V= \cup_{q\in U} S_q$. Let
    $x_{V_i}'$ be $x_V$ except for the $i\in V$-th coordinate; $x_i$ is changed to $x_i'$.
    Then
  \begin{align*}
    \sup_{x_V \in \X^{|V|},x_i'\in \X}\left|Z\left(x_V\right) - Z\left(x_{V_i}'\right)\right| \le 4, \, \forall i\in V.
  \end{align*}
\end{lemma}
\begin{proof} Since $(S_q)_{q\in [Q]}$ is a partition of $[Q]$, $(S_q)_{q \in U}$ forms a partition of $V$ and there exists a unique $r\in U$ such that $i \in S_r$.  Let
  \begin{align*}
    &Y_q := Y_q(f,x_V),\\  
    & q \in U = \phi\left(\frac{2\left<f,\mu_{S_q} -\mu_\P\right>_{\k}}{\epsilon}\right) - \E\phi\left(\frac{2\left<f,\mu_{S_q} -\mu_\P\right>_{\k}}{\epsilon}\right),\\
    &Y_r' := Y_r(f,x_{V_i}').
  \end{align*}
In this case
  \begin{align*}
    &\left|Z\left(x_V\right) - Z\left(x_{V_i}'\right)\right|\\
    &=
       \left| \sup_{f\in B_\k} \sum_{q \in U} Y_q    - \sup_{f\in B_\k} \left(\sum_{q \in U\backslash \{r\}} Y_q   + Y_r'   \right)\right|\\
       &\stackrel{(a)}{\le} \sup_{f\in B_\k } \left| Y_r- Y_r' \right|
    \stackrel{(b)}{\le} \sup_{f\in B_\k } \Big( \underbrace{\left| Y_r \right|}_{\le 2}  +  \underbrace{\left|Y_r'\right|}_{\le 2} \Big) \le 4,
  \end{align*}
  where in (a) we used Lemma~\ref{lemma:sup}, (b) the triangle inequality and the boundedness of $\phi$ [$|\phi(t)|\le 1$ for all $t$].

\end{proof}

\begin{lemma}[Uniformly bounded separable Carath{\'e}odory family] \label{eq:UBSC}
Let $\epsilon>0$, $N\in \Z^+$, $Q\in \Z^+$, $M=N/Q\in \Z^+$, $\phi(t) = (t-1) \I_{1\le t \le 2} + \I_{t\ge 2}$, $\k: \X \times \X \rightarrow \R$ is a continuous kernel on the separable topological domain $\X$,
$\mu$ is the mean embedding associated to $K$, $\P_M := \frac{1}{M} \sum_{m=1}^M\delta_{u_m}$, $\G = \{g_f: f\in B_{\k}\}$, where $g_f: \X^M \rightarrow \R$ is defined as
\begin{align*}
    g_f(u_{1:M}) &=  \phi\left(\frac{2\left<f,\mu_{\P_M}-\mu_{\P}\right>_{\k}}{\epsilon}\right).
\end{align*}
Then $\G$ is a uniformly bounded separable Carath{\'e}odory family:
    (i) $\sup_{f\in B_\k} \left\|g_f\right\|_{\infty} <\infty$ where $\left\|g\right\|_{\infty} = \sup_{u_{1:M}\in\X^M} |g(u_{1:M})|$,
    (ii) $u_{1:M}  \mapsto g_f(u_{1:M})$ is measurable for all $f\in B_\k$,
    (iii) $f\mapsto g_f(u_{1:M})$ is continuous for all $u_{1:M}\in \X^M$,
    (iv) $B_\k$ is separable.

\end{lemma}
\begin{proof}~
\begin{itemize}[labelindent=0.3cm,leftmargin=*,topsep=0cm,partopsep=0cm,parsep=0cm,itemsep=0cm]
 \item[(i)] $|\phi(t)|\le 1$ for any $t$, hence $\left\|g_f\right\|_{\infty} \le 1$ for all $f\in B_\k$.
  \item[(ii)] Any $f\in B_\k$ is continuous since $\Hk \subset C(\X)=\{h:\X \rightarrow\R\text{ continuous}\}$, so $u_{1:M} \mapsto (f(u_1),\ldots,f(u_M))$ is continuous.
    $\phi$ is Lipschitz, specifically continuous. The continuity of these two maps imply that of $u_{1:M} \mapsto g_f(u_{1:M})$, specifically it is Borel-measurable.
 \item[(ii)] The statement follows by the continuity of $f \mapsto  \left<f,h\right>_{\k}$ ($h=\mu_{\P_M}-\mu_{\P}$) and that of $\phi$.
 \item[(iv)] $B_\k$ is separable since $\Hk$ is so by assumption.
\end{itemize}
\end{proof}

\section{External Lemma} \label{sec:app:ext-lemmas}
Below we state the McDiarmid inequality for self-containedness.

\begin{lemma}[McDiarmid inequality]\label{lemma:McDiarmid}
Let $x_{1:N}$ be $\X$-valued independent random variables. Assume that $f: \X^N\rightarrow \R$ satisfies the bounded difference property
\begin{equation*}
\sup_{u_1,\ldots,u_N,u_r'\in \X}\left|f(u_{1:N}) - f(u_{1:N}')\right| \le c, \quad \forall n \in [N],
\end{equation*}
where $u_{1:N}' =\left( u_1,\ldots,u_{n-1},u_n',u_{n+1},\ldots,u_N\right)$.
Then for any $\tau>0$
\begin{align*}
    \P\left( f(x_{1:N}) < \E_{x_{1:N}}\left[f(x_{1:N})\right] +  c \sqrt{\frac{\tau N}{2}}\right) \ge 1-e^{-\tau}.
\end{align*}
\end{lemma}

\section{Pseudocode of Experiment-2}\label{sec:experiment-2}
The pseudocode of the two-sample test conducted in Experiment-2 is summarized in Algorithm~\ref{alg:DNA}.
\begin{algorithm} 

\begin{minipage}{\textwidth/2}
   \caption{Two-sample test (Experiment-2)}
   \label{alg:DNA}
    \begin{algorithmic}
    \STATE \tb{Input:} Two samples: $(X_n)_{n\in[N]}$, $(Y_n)_{n\in[N]}$. Number of bootstrap permutations: $B\in \Z^+$. Level of the test: $\alpha \in (0,1)$. Kernel function with hyperparameter $\theta \in \Theta$: 
      $\k_{\theta}$.
    \STATE Split the dataset randomly into 3 equal parts: 
	\vspace*{-0.3cm}
	\begin{align*}
	  [N] &= \mathop{\dot{\bigcup}}_{i=1}^3 I_i, \quad |I_1|=|I_2|=|I_3|.
	\end{align*}
	\vspace*{-0.3cm}
    \STATE Tune the hyperparameters using the 1st part of the dataset: 
	  \begin{align*}
	      \hat \theta &=\argmax_{\theta \in \Theta} J_{\theta}:= \widehat{\text{MMD}}_{\theta}\left((X_n)_{n \in I_1},(Y_n)_{n \in I_1}\right).
	  \end{align*}
    \STATE Estimate the $(1-\alpha)$-quantile of $\MMDhat_{\widehat{\theta}}$ under the null, using $B$ bootstrap permutations from  $(X_n)_{n\in I_2} \cup (Y_n)_{n\in I_2}$: 
	  $\hat{q}_{1-\alpha}$.
    \STATE Compute the test statistic on the third part of the dataset:
	\begin{align*}
	    T_{\hat{\theta}} &= \widehat{\text{MMD}}_{\widehat{\theta}}\left((X_n)_{n \in I_3},(Y_n)_{n \in I_3}\right).
	\end{align*}
	\STATE \tb{Output:} $T_{\hat{\theta}}-\hat{q}_{1-\alpha}$.
    \end{algorithmic}
\end{minipage}    
\end{algorithm}

\end{document}